\documentclass{article}

 \usepackage[preprint]{neurips_2026}

\usepackage[utf8]{inputenc} 
\usepackage[T1]{fontenc}    
\usepackage{hyperref}       
\usepackage{url}            
\usepackage{booktabs}       
\usepackage{amsfonts}       
\usepackage{nicefrac}       
\usepackage{microtype}      
\usepackage{xcolor}         

\usepackage{graphicx}
\usepackage{subcaption}
\usepackage{wrapfig}
\usepackage{amsmath}
\usepackage{amssymb}
\usepackage{mathtools}
\usepackage{amsthm}
\usepackage{algorithm}
\usepackage{algorithmicx}
\usepackage{algpseudocode}
\theoremstyle{plain}
\newtheorem{theorem}{Theorem}[section]
\newtheorem{proposition}[theorem]{Proposition}

\theoremstyle{definition}

\theoremstyle{remark}

\title{DG-CoLearn: An Efficient Collaborative Learning Framework for Dynamic Graphs}

\author{
Au Ashley Hoi-Ting$^{1}$ \quad Zikun Zhang$^{1}$ \quad Ligang He$^{1}$ \quad Qiang Ni$^{2}$ \\
$^{1}$Department of Computer Science, The University of Warwick \\
$^{2}$School of Computing and Communications, Lancaster University \\
\texttt{\{ashley.au, zikun.zhang.1, ligang.he\}@warwick.ac.uk} \\
\texttt{q.ni@lancaster.ac.uk}
}

\begin{document}

\maketitle

\begin{abstract}
  Dynamic graph learning (DGL) is essential for modelling evolving graph data, but existing methods suffer from significant computational overhead due to repeated full-snapshot retraining and are not well-suited for collaborative settings with partitioned data. In realistic graph systems, cross-partition edges are unavoidable, but direct sharing of graph structure between clients may violate privacy constraints. We propose DG-CoLearn, a client-oblivious collaborative dynamic graph learning framework built on incremental graph snapshot processing, which focuses computation on graph regions affected by temporal updates while preserving historical information through temporal modelling. This incremental design is consistently applied across the entire graph processing pipeline, including a server-mediated embedding exchange mechanism to enable accurate multi-hop message passing without exposing raw cross-client structural information. Extensive experiments demonstrate that DG-CoLearn achieves up to 33.8$\times$ speedup in training time and 27.4$\times$ reduction in communication overhead, while consistently improving predictive performance on both node classification (up to 13.36\% F1 improvement) and link prediction (up to 8.27\% MAP improvement) tasks. These results highlight the effectiveness of DG-CoLearn in bridging efficiency, scalability, and client-to-client structural privacy in collaborative dynamic graph learning.
\end{abstract}

\section{Introduction}

Dynamic graphs, where nodes, edges, and their features evolve over time, are prevalent in many real-world applications such as social networks, biology science and financial data analysis \cite{applications}. For example, social network platforms such as Instagram and Tiktok collect and send user interactions to a centralized server. These data naturally form a dynamic graph as new interactions are introduced. Learning meaningful representations from these evolving structures, commonly refered to Dynamic Graph Learning (DGL), is critical for tasks including link prediction and node classification. 
However, centralized DGL faces significant challenges. Frequent graph updates require repeated processing of large and evolving graph snapshots, leading to prohibitive computational and memory cost on the server. 

Collaborative learning offers a promising solution by distributing computation across multiple clients \cite{FederatedScope, FedSage}. In graph settings, this typically takes the form of \textit{intra-graph distributed learning}, where a central server partitions data into client-specific subgraphs and coordinates distributed learning. While this paradigm improves scalability, it introduces a unique tension between efficiency, structural dependency, and privacy. In particular, real-world graphs are inherently interconnected: cross-partition edges are unavoidable and carry essential information for accurate learning. At the same time, privacy constraints often prohibit clients from accessing cross-partition connectivity, creating incomplete local views and hindering effective message passing.


Existing approaches to collaborative or federated graph learning primarily focus on static graphs and assume either (i) disjoint client subgraphs without cross-client edges, or (ii) full visibility of cross-client connectivity \cite{fedgraph, llcg}. These assumptions limit their applicability in realistic dynamic settings. Moreover, prior work largely adopts the classical federated learning (FL) privacy model, where the central server is treated as untrusted and the primary objective is to prevent data leakage to the server \cite{MIA1}. However, this model does not capture an important class of real-world systems, such as financial platforms or infrastructure networks, where a central coordinator necessarily maintains global topology, and the primary privacy risk lies in preventing \textbf{information leakage between clients} rather than from clients to the server.


In this work, we formalize and study this alternative setting, which we term \textbf{client-oblivious collaborative graph learning}. In this regime, a (trusted or secured) central coordinator maintains the global graph structure and mediates communication, while clients are restricted to their local subgraphs and must remain oblivious to cross-client connectivity. The key privacy objective is thus \textit{horizontal} privacy: ensuring that no client can infer the structure, features, or labels of another client’s subgraph, even in the presence of cross-partition dependencies. This setting introduces new challenges that are not addressed by existing FL frameworks, particularly in dynamic graphs where structural changes are frequent and localized.

Additionally, a fundamental inefficiency arises in existing dynamic graph learning methods, both centralized and collaborative, which typically retrain models over entire graph snapshots at each timestamps. This design is particularly wasteful in evolving graphs, where updates are often sparse and localized. In collaborative settings, this inefficiency is further amplified by repeated cross-client communication and redundant recomputation, making existing approaches impractical at scale.

To address these challenges, we propose \textbf{DG-CoLearn}, a collaborative learning framework for dynamic graphs built on a unifying principle: \textbf{incremental snapshot processing}. Instead of retraining on full graph snapshots, DG-CoLearn selectively updates only the nodes and edges affected by temporal changes, while preserving historical representations for unchanged regions. This principle permeates all stages of the framework: graph partitioning is restricted to affected regions, clients perform lightweight local training on updated subgraphs, and cross-client node embedding computation are limited to newly influenced nodes.


Building on this design, DG-CoLearn achieves both efficiency and effectiveness. By exploiting temporal locality, it significantly reduces computation and communication overhead, while its privacy-aware aggregation mechanism enables accurate modelling of cross-client dependencies. Extensive experiments on multiple dynamic graph benchmarks demonstrate that DG-CoLearn achieves substantial speedups over full-snapshot retraining while consistently improving predictive performance on both link prediction and node classification tasks.
Our contributions are summarized as follows:
\begin{itemize}
    \item We formalise client-oblivious collaborative graph learning, a practically motivated privacy regime that focuses on preventing cross-client structural leakage under a centralized coordinator.
    \item We propose a novel duo-stage graph partitioning algorithm tailored to collaborative graph learning, reducing cross-client communication while maintaining balanced and informative subgraphs for effective and generalizable training on both node classification and link prediction tasks.
    \item We incorporate a lightweight local training scheme for clients to incrementally perform dynamic graph learning while preserving historical knowledge.
    \item We design a server-mediated cross-client node knowledge exchange mechanism that derives global representations without raw data exchange.
\end{itemize}

\section{Related Work} \label{sesh:relatedwork}

\textbf{Dynamic Graph Learning.} Centralized DGL has been studied through both discrete-snapshot models \cite{EvolveGCN, DMGCN, DyGFormer} and continuous-time (CT) formulations such as TGAT \cite{TGAT}, TGN \cite{TGN}, and DyGFormer \cite{DyGFormer}, which maintain per-node memory states and attend over temporal neighbourhoods. Among snapshot-based approaches, ROLAND \cite{ROLAND} decouples temporal modelling from graph structure by hierarchically updating node states across snapshots via a GRU, avoiding sequential retraining over all historical snapshots. DG-CoLearn builds on ROLAND's state-update perspective but further observes that, if node states are updated incrementally, unchanged nodes need not be retrained at all, \textbf{motivating our incremental snapshot processing under collaborative and privacy-constrained settings}. CT methods are complementary: their backbone-agnostic encoders can be composed as local models within DG-CoLearn, as demonstrated in Section \ref{sesh:experi_result}. From a systems perspective, BLAD and ESDG \cite{BLAD, ESDG} reduce communication by grouping snapshots onto single GPUs, but this assumes batch access to multiple snapshots and scales poorly as individual snapshots grow large, limiting applicability to streaming or rapidly evolving graphs.

\textbf{Collaborative and Federated Graph Learning.} Early federated graph methods such as FedSage, FedEgo, and FedGCN \cite{FedSage, FedEgo, FedGCN} primarily focus on static graphs, and either assume disjoint client graphs or require explicit sharing of structural information. FedDGL \cite{FedDGL} extends FL to dynamic graphs via knowledge distillation but inherits the disjoint-subgraph assumption, ignoring cross-client edges entirely. Methods that do address cross-client dependencies, FedSage+ \cite{FedSage} via missing-neighbour generation, LLCG \cite{llcg} via server-side global correction, recover structural signal but rely on expensive generative reconstruction or repeated full-graph processing, neither of which scales to dynamic graphs with frequent updates.

\textbf{Privacy in Federated Graph Learning.} Privacy preservation in FL has primarily been studied under untrusted-server assumptions, with secure aggregation and differential privacy aimed at preventing the server from recovering individual updates. However, gradient inversion attacks \cite{GIA1, GIA2, GIA3} have shown that training data can still be reconstructed from shared gradients, and membership inference attacks have been extended to graph neural networks \cite{MIA2, MIA4, MIA1, MIA3}, where node or edge membership can be inferred from model outputs or embeddings. While differentially private federated graph learning frameworks \cite{DPDGL} introduce noise or restrict information sharing to protect client data, they do not explicitly consider cross-client structural leakage in partitioned graph settings. In contrast, DG-CoLearn considers a different privacy axis, preventing \textbf{client-to-client} structural inference under a trusted coordinator, which is practically motivated by systems where a central operator necessarily maintains global topology and the primary risk lies between clients rather than between clients and server.

\textbf{Graph Incremental Learning.} Recent studies distinguishes graph incremental learning (GIL) from dynamic graph learning (DGL) \cite{ngil}, and we maintain this distinction. GIL methods such as DyGRAIN, TWP, and ER-GNN \cite{DyGRAIN, TWP, ER-GNN} target catastrophic forgetting as new graph data arrives, whereas DGL, the focus of this work, aims to capture temporal dynamics. Although mitigating catastrophic forgetting is not our central objective, we additionally evaluated its ability to retain historical knowledge in Appendix \ref{appen:catforget}.

\section{Problem Formulation} \label{sesh:formulation}

\subsection{Dynamic Graphs} 
We consider a dynamic graph dataset $\mathcal{D} = \{G_t\}_{t=1}^T$ , consisting a sequence of temporal graph snapshots. At each timestamp $t$, the snapshot is defined as $$G_t = (V_t, E_t, X_t, Y_t),$$
where $V_t$ and $E_t$ are the set of nodes and edges respectively, $X_t \in \mathbb{R}^{|V_t| \times d}$ represents the raw node features, and $Y_t$ is the available node labels for the learning task. The graph topology is encoded by an adjacency matrix $\mathbf{Adj_t} \in \{0,1\}^{|V_t|\times |V_t|}$. 

Following the temporal graph formulation in \cite{temporalpart}, we define the delta graph $$\Delta G_{(t,t+1)} = (\Delta V_{(t,t+1)}, \Delta E_{(t,t+1)}),$$
where $\Delta V_{(t,t+1)}$ contains newly arrived nodes between timestamps $t$ and $t+1$, and $\Delta E_{(t,t+1)}$ contains both new edges connecting existing nodes, and edges connecting two new vertices.

\subsection{Collaborative Intra-graph learning}
We study a collaborative intra-graph learning paradigm, where
a global graph at each timestamp is maintained by a central server and partitioned into multiple subgraphs, which are distributed to a set of clients. Each client holds one subgraph and performs local GNN training.


At each timestamp $t$, the server partitions the global node set $V_t$  into $M$ disjoint client subgraphs:
$$
V_t = \bigcup_{m=1}^M V_t^{(m)}, \qquad 
V_t^{(i)} \cap V_t^{(j)} = \emptyset \text{ for } i \neq j,
$$
with corresponding subgraphs
$$
G_t^{(m)} = \big(V_t^{(m)}, E_t^{(m)}, X_t^{(m)}, Y_t^{(m)}\big),
$$
where $E_t^{(m)} = \{(u,v) \in E_t |\ u,v \in V^{(m)}_t\}$ contains only the edges for which both endpoints lie in $V_t^{(m)}$. The corresponding local adjacency matrix held by client $m$ is $\mathbf{Adj_t}^{(m)} \in \{0,1\}^{|V_t^{(m)}| \times |V_t^{(m)}|}$, which encodes only intra-client edges.

To enable message passing across partitions, a central coordinator maintains global structural information and facilitates controlled information exchange between clients.

The learning objective is to jointly optimize local models across clients while preserving performance comparable to centralized training under computational and communication constraints.

\subsection{Privacy Model and Threat Assumptions}
DG-CoLearn operates under a \emph{client-oblivious trusted-server} threat model.
The server holds the full adjacency matrix $\mathbf{Adj_t}$ and orchestrates partitioning, model aggregation, and cross-client node embedding exchange. It is assumed to execute the protocol faithfully and does not attempt to infer private client data beyond what is required for aggregation. Each client $m$ observes only its local subgraph $G_t^{(m)}$, the aggregated global model parameters, and server-computed embedding corrections for its border nodes. Crucially, any cross-client edges connecting $V_t^{(k)}$ to $V_t^{(j)}$ ($j \neq k$) are known exclusively to the server. Clients are unaware of whether their local nodes connect to nodes owned by other clients, and must train using incomplete neighbourhood information.


The privacy goal is \emph{client-to-client structural privacy}: no client should be able to reconstruct or infer the subgraph structure, node features, or labels of any other client from the information it receives. Appendix~\ref{appen:privacy} provides a formal analysis showing that the server-computed embedding corrections received by each client are provably insufficient for reconstructing another client's private subgraph, even under multi-client collusion. This regime differs from classical federated learning, which targets an untrusted server via secure aggregation or differential privacy, techniques that do not directly prevent inference of cross-client graph structure under partitioned graphs with shared coordination.

\section{The DG-CoLearn Framework}

At each timestamp $t$, DG-CoLearn executes a three-stage pipeline:

\textbf{(i) Incremental graph partitioning (§\ref{sesh:gpa})}: the server applies \textit{CoLearnPartition} to assign newly arriving nodes while preserving historical partitions, minimizing cross-client edges and maintaining balanced, learning-effective subgraphs.

\textbf{(ii) Lightweight local training (§\ref{sesh:only_learn_new})}: each client update its GNN model using only a $k$-hop subgraph induced by locally affected nodes, while a temporal module preserves historical representations without reprocessing unchanged regions.

\textbf{(iii) Server-mediated node knowledge aggregation (§\ref{sesh:exchange_ne})}: clients submit updated local models together with embeddings of boundary nodes involved in cross-client dependencies; the server incrementally reconstructs exact multi-hop node embeddings using global topology without exposing cross-client connectivity.

By explicitly exploiting temporal locality and structural sparsity of updates, DG-CoLearn enables scalable, privacy-preserving collaborative learning on evolving graphs, avoiding full-snapshot retraining and expensive generative reconstruction while achieving centralized-learning accuracy.

\subsection{CoLearnPartition: Duo-Stage Temporal Graph Partitioning Algorithm} \label{sesh:gpa}

To support scalable collaborative learning on evolving graphs, DG-CoLearn requires a partitioning algorithm that (i) operates incrementally over temporal snapshots, (ii) preserves partition consistency over time, and (iii) facilitates effective downstream GNN training under client-to-client privacy constraints. Existing temporal partitioning methods \cite{temporalpart} primarily optimize structural objectives such as minimizing edge cuts or balancing partition size, but do not explicitly account for learning effectiveness or computation cost in graph partitioning.

We therefore propose \textbf{CoLearnPartition}, a lightweight duo-stage incremental vertex partitioning algorithm tailored for collaborative DGL. CoLearnPartition assigns newly arriving nodes while keeping historical assignments fixed, avoiding costly repartitioning at each timestamp. It jointly optimizes three objectives: (1) minimizing cross-client edges, to reduce embedding exchange overhead; (2) balancing partition edge volume, to ensure comparable computational load across clients; and (3) balancing node label distributions, to promote stable and generalizable collaborative training.

CoLearnPartition enforces these objectives through a two-stage decision process: a \textit{primary stage} that applies a hard constraint on a dominant objective, followed by \textit{a secondary refinement stage} that softly optimizes the remaining criteria. We empirically evaluate different stage configurations in Appendix \ref{appen:gpa_experiments} and find that prioritizing cross-client edge minimization in the primary stage, followed by the other two objectives in the secondary stage, yields the most effective trade-off between communication efficiency and learning performance. This design enables efficient partition updates that directly align partition structure with collaborative learning efficiency.

At each snapshot $t$, CoLearnPartition searches for a valid partition from snapshot $t-1$, and incrementally assign newly arrived nodes $\Delta V_{(t-1,t)}$ based on their 1-hop and 2-hop neighbourhood to the existing partition. Otherwise, the graph will be partitioned according to Appendix \ref{appen:gpa}. Note that the border nodes, defined as nodes with neighbours in other subgraphs after the initial partition, play a key role in cross-partition interactions.
Unlike snapshot partitioning \cite{BLAD}, which segments the graph by time and creates redundant copies, our vertex partitioning preserves enables fine-grained, communication-efficient parallelism and better supports node-centric collaborative learning tasks. The detailed explanation of CoLearnPartition is provided in Appendix \ref{appen:gpa}.

\subsection{Lightweight Dynamic GNN Learning} \label{sesh:only_learn_new}

As discussed earlier in Section \ref{sesh:relatedwork}, our design is inspired by ROLAND, which models temporal dependencies by updating node states across snapshots using a GRU. This formulation implies that historical node states are already preserved in the recurrent state. We therefore observe that, especially when graph updates are sparse, it is unnecessary to retrain over the entire snapshot.

\begin{figure*}[t]
  \centering
  \includegraphics[width=\textwidth]{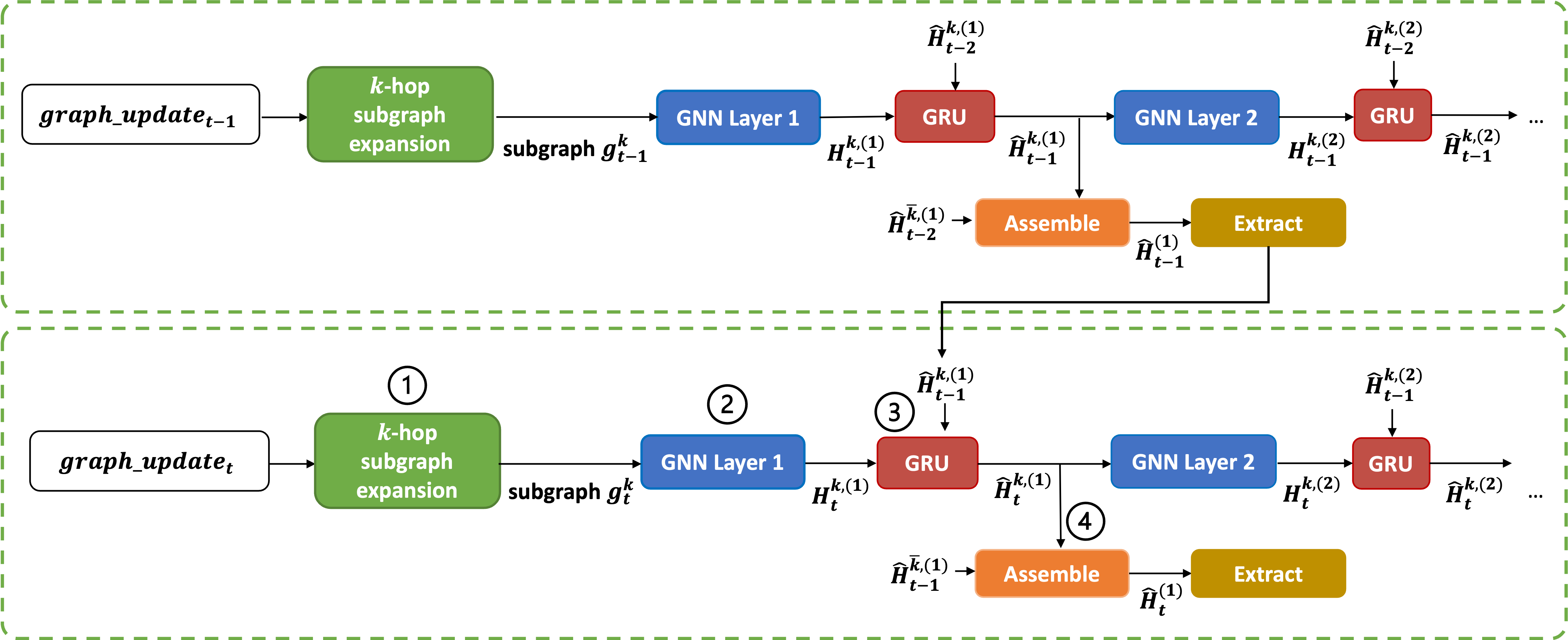}
  \caption{Architecture of DG-CoLearn. The key four steps are labelled on the processing of snapshot $t$ (bottom dotted box).}
  \label{fig:dg-colearn}
\end{figure*}

DG-CoLearn enables each client to perform \textbf{incremental local training}, where the learning is restricted to subgraphs induced by affected nodes while integrating historical representations for unchanged regions. Figure \ref{fig:dg-colearn} illustrates the architecture which proceeds in four steps:

\begin{enumerate}
    \item \textbf{$k$-hop subgraph expansion.} Given the delta graph $\Delta G_{(t-1,t)}$, the client extracts the $k$-hop induced subgraph \begin{equation}\label{eq:learning1}
    g^k_t = G_t[\mathcal{V}(g_t^k) ], \quad  \mathcal{V}(g_t^k) = \bigcup_{v \in \Delta V_{(t-1,t)}} \mathcal{N}_{G_t}^k(v),
    \end{equation} where $k$ matches the number of GNN layers and $\mathcal{N}_{G_t}^k(v)$ denotes the $k$-hop neighbourhood of node $v$ in $G_t$. This restricts computation to nodes whose representations can be affected by the update.
    
    \item \textbf{GNN layer.} $g^k_t$ is passed through the $l$-th GNN layer. For each node $v \in g^k_t$, the layer computes \begin{equation} \label{eq:learning2}
    \scalebox{0.9}{$
    H^{k,(l)}_{t,v} = \sigma\left( W^{(l)} \cdot \text{AGG}^{(l}\left( \left\{ H^{(l-1)}_{t,u} \mid u \in \mathcal{N}(v) \cup {\{v\}} \right\} \right) \right)
    $}\end{equation} and $\mathbf{H}^{k,(l)}_t = \bigcup_{v \in g^k_t} H^{k,(l)}_{t,v}$ collects the resulting embeddings.

    \item \textbf{GRU temporal update.} A GRU captures the temporal relationships between consecutive graph updates by fusing the new layer-$l$ embeddings with the previous snapshot's hidden states for the nodes in $g^k_{t}$: \begin{equation} \label{eq:learning3}
    \widehat{\textbf{H}}^{k,(l)}_t = \text{GRU}\left( \textbf{H}^{k,(l)}_t, \textbf{H}^{k,(l)}_{t-1} \right)
    \end{equation}

    \item \textbf{Assembling the full hidden states.} Since nodes outside $g^k_{t}$ are unchanged, their hidden states from $t-1$ (i.e. $\widehat{\textbf{H}}^{\overline{k},(l)}_{t-1} = \widehat{\textbf{H}}^{(l)}_{t-1} \setminus \widehat{\textbf{H}}^{k,(l)}_{t-1}$) are directly reused: 
    \begin{equation} \label{eq:learning4}
    \widehat{\textbf{H}}^{(l)}_t = \widehat{\textbf{H}}^{k,(l)}_t \cup \widehat{\textbf{H}}^{\overline{k},(l)}_{t-1}
    \end{equation}
\end{enumerate}

These steps are repeated $k$ times until subgraph $g^k_{t}$ has been processed through all $k$ GNN layers. Full per-step equations and aggregation details are provided in Appendix \ref{appen:learning}.

DG-CoLearn improves upon centralized baselines by enabling localized updates and reducing computation via decentralized, $k$-hop subgraph processing. Compared to knowledge distillation methods such as DyGRAIN \cite{DyGRAIN}, which rely on full-graph access to mitigate forgetting, DG-CoLearn preserves historical information through node state reuse and incrementally learns newly updated graph elements. As a result, DG-CoLearn is a better suit for real-world, resource-constrained collaborative settings.

\subsection{Exchanging Node Knowledge across Clients} \label{sesh:exchange_ne}

In collaborative intra-graph learning, clients train GNNs on disjoint subgraphs and are unaware of cross-client dependencies. While this preserves privacy, it prevents clients from accurately incorporating information from cross-client neighbours, leading to incomplete message passing. DG-CoLearn addresses this challenge through a \textbf{server-mediated node knowledge exchange mechanism} that reconstructs \textit{exact} embeddings of affected nodes without revealing cross-client edges to clients.

\begin{wrapfigure}{r}{0.25\textwidth}
  \centering
  \includegraphics[width=\linewidth]{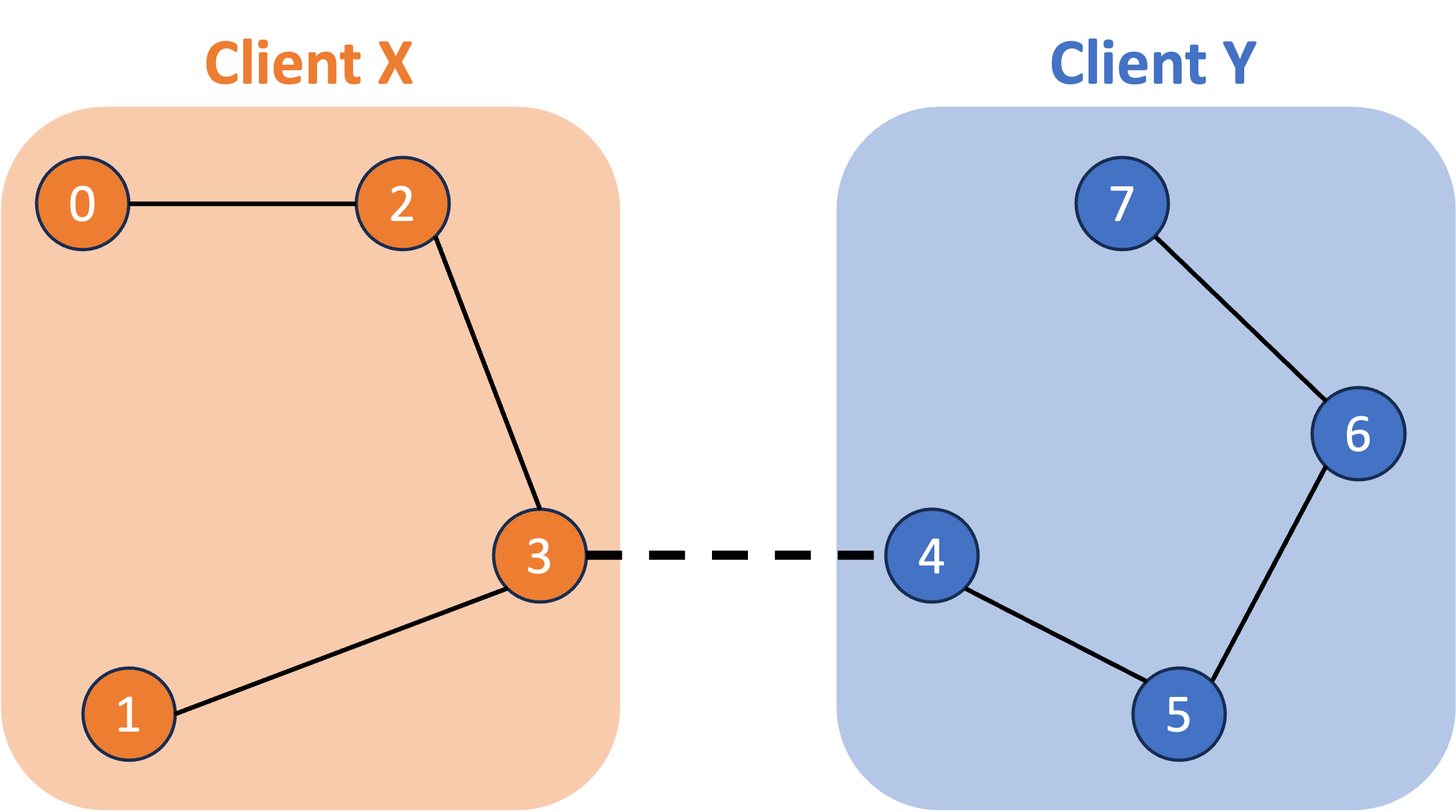}
  \caption{Example graph with a cross-client edge $(3,4)$, where $\widehat{\mathcal{N}}_3 = \{4\}$ and $\widehat{\mathcal{N}}_4 = \{3\}$}
  \label{fig:cce}
\end{wrapfigure}
Figure \ref{fig:cce} illustrates a cross-client edge connecting node 3 (in client X) and node 4 (in client Y). We denote $\widehat{\mathcal{N}}_i$ as the set of nodes that are the cross-client neighbours of node $i$.

Crucially, DG-CoLearn performs node embedding exchange \textbf{incrementally.} At snapshot $t$, only nodes in the $k$-hop induced subgraph $g_t^k$ defined in Equation \ref{eq:learning1} require embedding correction. This incremental exchange strategy reduces communication complexity from $O(|V_t|)$ per snapshot to $O(|\Delta V_{(t-1,t)}|)$, while preserving message passing accuracy. 

FedGCN \cite{FedGCN} reconstructs cross-client information by summing locally aggregation node features across clients. While being communication efficient, this approach approximates message passing and overlooks interactions among cross-client neighbours, leading to inaccurate multi-hop embeddings.


\subsubsection{Computing exact node embeddings}
From Equation \ref{eq:learning2}, we denote $H^{k,(l)}_{t,v}$ as node $v$'s embedding after the $l^{th}$ GNN layer.

As discussed in FedGCN \cite{FedGCN}, communicating with $(L>2)$-hop neighbours is infeasible because of potential data leakage and the coverage of the entire graph, incurring prohibitive communication and computation costs. Thus, we restrict ourselves to 2-hop communication as well and derive the method for computing the 1-hop and 2-hop node embeddings.

The server, aware of the cross-client edges, computes and redistributes \textbf{exact} 1-hop and 2-hop embeddings for each node via the following theorem. The complete proof of the theorem will be provided in Appendix \ref{appen:proof}.

\begin{theorem} \label{thm:thm1}
    The exact 1-hop global node embeddings for node $i$, denoted by $H^{*,(1)}_t(i)$, can be computed by:
    \begin{equation} \label{eq:glob1}
        H^{*,(1)}_{t,i} = H^{k,(1)}_{t,i} + \sum_{j \in \widehat{\mathcal{N}}_i} H^{(0)}_{t,j}
    \end{equation}
    and the 2-hop global node embeddings $H^{*,(2)}_t(i)$ of node $i$ can be computed by:
    \begin{equation} \label{eq:glob2}
\begin{aligned}
H^{*,(2)}_{t,i}
&=
H^{k,(2)}_{t,i}
+
\sum_{j \in \widehat{\mathcal{N}}_i}
\big(
H^{k,(1)}_{t,j}
+
H^{(0)}_{t,j}
+
H^{(0)}_{t,i}
\big) +
\sum_{q \in \widehat{\mathcal{N}}_j}
H^{(0)}_{t,q},
\end{aligned}
\end{equation}

where $H^{(0)}_{t,v}$ as its initial node feature of node $v$.

\end{theorem}

Using the derived formulae, we proposed to have clients share their 1-hop and 2-hop node embeddings to the server to obtain the exact node embeddings for their own cross-client nodes. Figure \ref{fig:ne-exchange}a illustrates the four stages of the proposed scheme. The reconstructed embeddings are then privately returned to the original client as a message structured in Figure \ref{fig:ne-exchange}b. Specifically, the formulae encode the additional node embeddings required from other clients for each cross-client node $i$ in addition to their locally learnt embeddings (i.e. $H^{*,(1)}_{t,i}-H^{k,(1)}_{t,i}$ and $H^{*,(2)}_{t,i}-H^{k,(2)}_{t,i}$).

\begin{figure*}[t]
  \centering
  \includegraphics[width=1.0\textwidth]{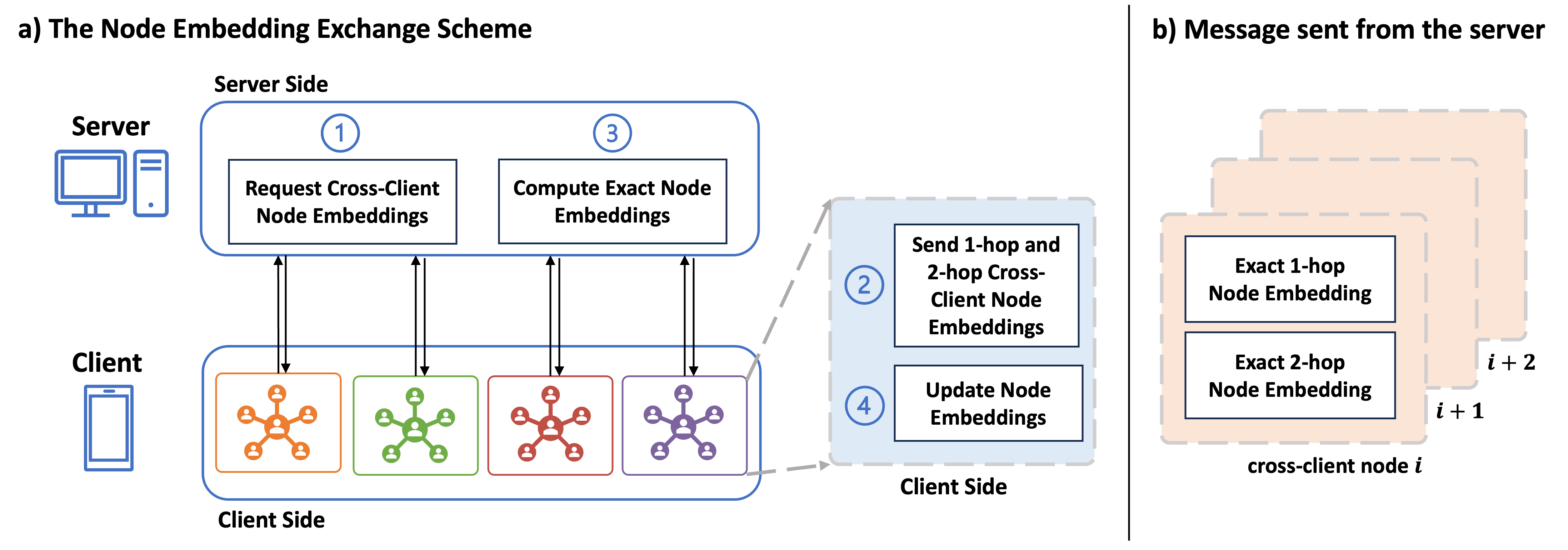}
  \caption{The Node Embedding Exchange Scheme. Figure on the left shows the four stages of the scheme, and the figure on the right shows the construction of a message from the server to a client.}
  \label{fig:ne-exchange}
\end{figure*}

Concretely, the additional embedding needed to obtain the exact embedding for node $i$ is $\sum_{j \in \widehat{\mathcal{N}}_i} H^{(0)}_{t,j}$, which the server computes by summing the 0-hop node embeddings of node $i$'s cross-client neighbours. 
The 2-hop correction is constructed analogously by combining the 1-hop embeddings  $H^{k,(1)}_{t,j}$ submitted by clients holding $j \in \widehat{\mathcal{N}}_i$ with the relevant 0-hop features (full derivation in Appendix \ref{appen:proof}). Because clients only receive aggregated sums returned to them by the server, no client observes another client's raw embeddings, features, or neighbourhood identities, a property we formulated in Appendix \ref{appen:privacy}.


To ensure training stability and enhance performance, we explore several synchronisation schedules for when the embedding exchange occurs. Exchanging only once in the last epoch after the first communication round provides a strong accuracy-privacy trade-off, with results summarized in Appendix \ref{appen:ne_experi}.



\section{Experiments}

\subsection{Experimental Settings}
We evaluate DG-CoLearn on two tasks: link prediction (via dot product of learned embeddings) and node classification, across six datasets spanning social networks (UCI, OTC), citation graphs (DBLP3, DBLP5), online communities (Reddit), and a large scale communication network (AS-733). To assess scalability beyond standard benchmarks, we additionally evaluate on two large-scale temporal graphs from the Temporal Graph Benchmarks \cite{huang2024tgb2}: tgbl-coin (link prediction on cryptocurrency transactions, 22M edges) and tgbn-reddit (node classification on user-subreddit interactions, 27M edges). For link prediction we report MRR and MAP following ROLAND \cite{ROLAND}; for node classification we report F1 following FedDGL \cite{FedDGL}. Baselines span centralized DGL (link prediction) and collaborative DGL (node classification); CoLearnPartition is additionally compared against Label-skew, Louvain, and Metis. Full experimental details are in Appendix \ref{appen:exp_detail}.





\subsection{Experimental Results} \label{sesh:experi_result}
\begin{table*}
\centering
\small
  \caption{DGL Models performances on Link Prediction Tasks}
  \label{tab:LP}
  \begin{tabular}{p{0.15\linewidth}ccc|ccc|ccc}
    \toprule
     & \multicolumn{3}{c|}{\textbf{MRR}} & \multicolumn{3}{c|}{\textbf{MAP}} & \multicolumn{3}{c}{\textbf{Accuracy}} \\
     Models& OTC & UCI & AS-733 & OTC & UCI & AS-733 & OTC & UCI & AS-733 \\
    \midrule
    GCRN-Baseline& 0.2000& 0.0783& O.O.M& 0.9183& 0.6837& O.O.M& 0.8337&0.7931&O.O.M\\
    GCRN-GRU& O.O.M& 0.0800& O.O.M& O.O.M& 0.7343& O.O.M& O.O.M&0.7610&O.O.M\\
    GCRN-LSTM& O.O.M& 0.0830& O.O.M& O.O.M& 0.7891& O.O.M& O.O.M&0.8112&O.O.M\\
    EvolveGCN-H & 0.1241& 0.0995& O.O.M& 0.6692& 0.6715& O.O.M& 0.5048&0.5985&O.O.M\\
    EvolveGCN-O & 0.0850& 0.0094& O.O.M& 0.5032& 0.5230& O.O.M& 0.5056&0.4991&O.O.M\\
    ROLAND & 0.1940& 0.1120& 0.3390& 0.8752 & 0.6967 & 0.9546&0.8660&0.8183&0.8748\\
    Ours & \textbf{0.7830} & \textbf{0.6922}& \textbf{0.8186}& \textbf{0.9353}& \textbf{0.8544}&\textbf{0.9566}&\textbf{0.9877}&\textbf{0.8372}&\textbf{0.9083}\\
    \midrule
    Improvement & 292\% & 518\% & 142\% & 1.85\%& 8.27\%& 0.210\%& 14.1\% & 2.31\%& 3.83\%\\
    \bottomrule
  \end{tabular}
\end{table*}

\textbf{Overall Performance.} 
Table \ref{tab:LP} and \ref{tab:NC} show that DG-CoLearn outperforms all baselines across both tasks and all datasets, achieving up to 518\% MRR and 8.27\% MAP improvement on link prediction, and up to 13.36\% F1 improvement on node classification. The gains over centralised DGL methods (Table \ref{tab:LP}) confirm that collaborative learning, when combined with our incremental learning, captures localized temporal patterns that single-machine approaches miss; the gains over federated baselines (Table \ref{tab:NC}) confirm that explicitly modelling cross-client edges via server-mediated exchange materially improves accuracy over disjoint-graph methods such as FedDGL.


\begin{table}[t]
\centering
\small
  \caption{FL Graph Learning Models F1 performances on Node Classification Tasks}
  \label{tab:NC}
  \begin{tabular}{p{3cm}ccc}
    \toprule
     & \textbf{DBLP5} & \textbf{DBLP3} & \textbf{Reddit} \\
    \midrule
    \multicolumn{4}{c}{Baseline Models} \\
    \midrule
    D-FedGNN & 0.6787& 0.5188&0.2033\\
    FedAvgDyn & 0.6733& 0.5381&0.2353\\
    FedSageDyn & 0.6856& 0.5439&0.2506\\
    FedProtoDyn & 0.7099& 0.5648&0.2487\\
    FedDGL & 0.7243& 0.5836& 0.2733\\
    \midrule
    \multicolumn{4}{c}{Graph Partitioning Algorithm with DG-CoLearn} \\
    \midrule
    Label-skew& 0.6852& 0.6368&0.2620\\
    Louvain& 0.7079& 0.6396&0.2618\\
    Metis& 0.7131 & 0.5966 &0.2709\\
    \midrule
    Ours &\textbf{0.7269}&\textbf{0.6616}&\textbf{0.2834}\\
    \midrule
    Improvement to best & 0.359\% & 13.36\% & 3.70\%\\
    \bottomrule
  \end{tabular}
\end{table}



\begin{figure}[t]
\begin{minipage}{0.49\linewidth}
\includegraphics[width=\linewidth]{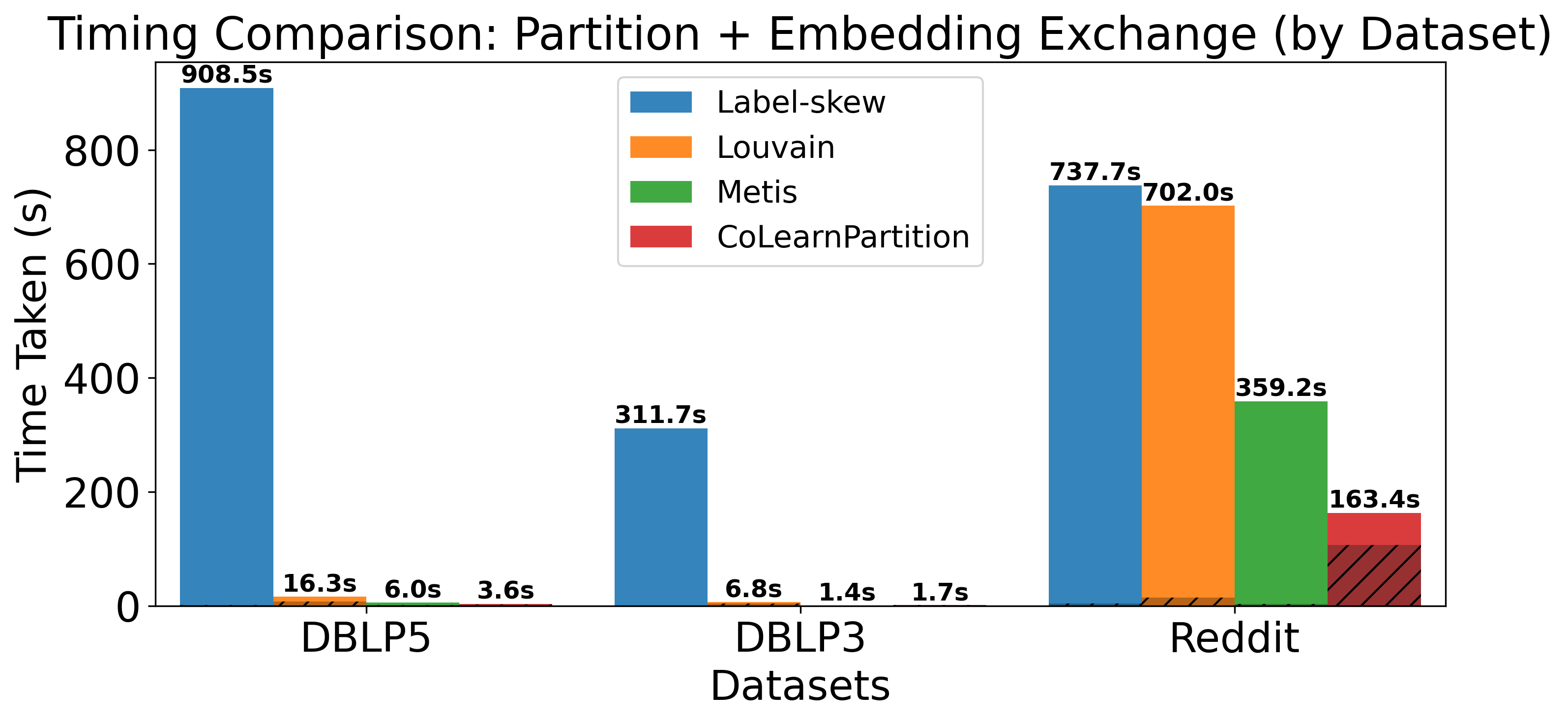}
\caption{Partition + embedding exchange time across partitioning methods.}
\label{fig:time_gpa}
\end{minipage}\hfill
\begin{minipage}{0.49\linewidth}
\includegraphics[width=\linewidth]{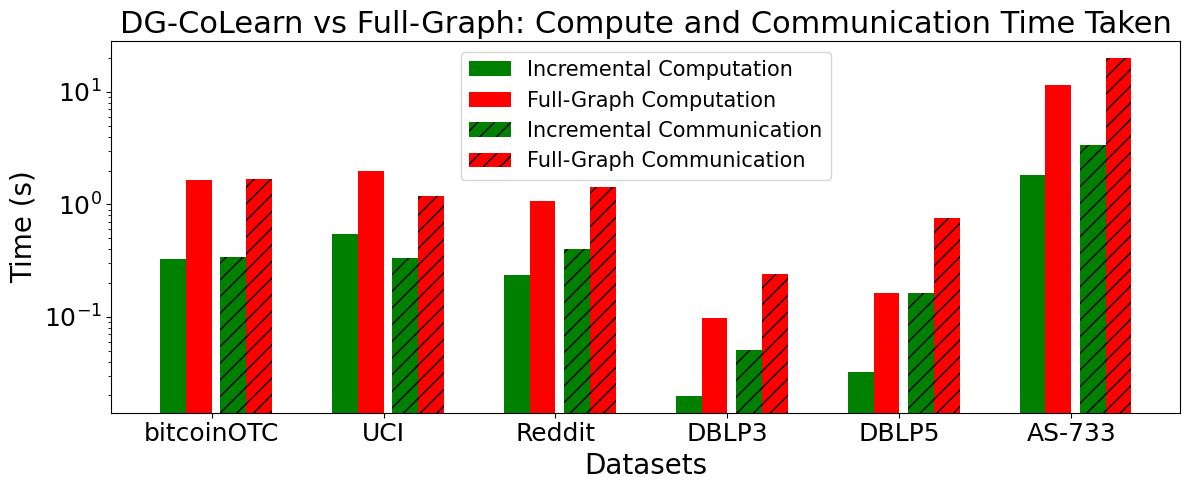}
\caption{Per-snapshot embedding exchange time: incremental vs.\ full-graph.}
\label{fig:time_ne}
\end{minipage}
\end{figure}

\textbf{Evaluating graph partitioning algorithms.} 
The bottom of Table \ref{tab:NC} shows CoLearnPartition yields up to \textbf{4.61\% F1 improvement} to the strongest baseline by producing subgraphs better suited for downstream learning. 
Figure \ref{fig:time_gpa} further shows it significantly reduces node embedding exchange time by up to 2.19× compared to the best algorithm Metis, a direct consequence of its cut-edge minimization objective, which dominates communication cost.


\textbf{Ablation study.} We isolate the contributions of (i) node embedding exchange and (ii) incremental processing. For (i), we evaluate link prediction restricted to cross-client node pairs against two baselines: no exchange and Incremental FedGCN \cite{FedGCN}, adapted to DGL by replacing input features with per-client feature aggregations (Appendix E). Table \ref{tab:sharing_perform} shows DG-CoLearn improves cross-client MAP by up to 11.7\% over the strongest baseline, since FedGCN's feature replacement breaks higher-order message passing while our scheme preserves exact multi-hop embeddings. For (ii), Figure \ref{fig:time_ne} shows our incremental embedding exchange achieves up to 6.3× speedup in computation and 5.9× in communication over full-graph exchange across all six datasets.




\begin{table}
\centering
\small
  \caption{Average Test MAP on Cross-Client Edges across Node Embedding Exchange Schemes}
  \label{tab:sharing_perform}
  \begin{tabular}{cccc}
    \toprule
    \textbf{Exchange Scheme}  &\textbf{OTC}& \textbf{UCI} & \textbf{AS-733}\\
    \midrule
     No exchange & 0.8540 & 0.7637 & 0.8192 \\
     Incremental FedGCN &0.8955 & 0.8493 & 0.9055\\
    DG-CoLearn & \textbf{1.000} & \textbf{0.9276} & \textbf{0.9275} \\
    \midrule
    Improvement to best & 11.7\% & 9.22\% & 2.43\% \\
    \bottomrule
  \end{tabular}
\end{table}

\textbf{Backbone compatibility and large-scale scalability.} We evaluate DG-CoLearn on two large-scale benchmarks: tgbl-coin and tgbn-reddit (Results in Appendix \ref{appen:tgb}). On tgbl-coin , DG-CoLearn adapts CT method DyGFormer and achieves a 33.8× training speedup, 51.6× partitioning speedup, and 27.4× reduction in node embedding exchange overhead, while the \textbf{full-graph retraining baseline fails to complete} under the same setting. On tgbn-reddit, DG-CoLearn delivers a 2.19× snapshot processing speedup and a 76.8\% reduction in cut-edge ratio compared to FedDGL, while also improving F1 by 11.98\%, a gain we attribute to recovering cross-client structural signal that FedDGL's disjoint-subgraph assumption discards. These results show that DG-CoLearn remains effective in regimes where standard distributed baselines do not complete at all.

\section{Conclusion}
We introduced \textbf{DG-CoLearn}, a collaborative learning framework for dynamic graphs built around \textit{incremental snapshot processing}. By systematically restricting computation to updated graph regions, DG-CoLearn applies this principle across partitioning, local training, and cross-client embedding reconstruction. We additionally formalise the client-oblivious privacy regime, practically motivated by systems where a trusted coordinator necessarily maintains global topology and the primary risk lies between clients, and prove that our server-mediated exchange mechanism is provably insufficient for cross-client subgraph reconstruction even under collusion. Extensive experiments on node classification and link prediction tasks demonstrate that DG-CoLearn achieves significant efficiency gains while maintaining or improving predictive performance. These results suggest that the apparent tension between client-to-client privacy and accuracy in collaborative DGL is not fundamental, but an artifact of treating each challenge independently rather than designing for them jointly.

\small
\bibliographystyle{abbrv}
\bibliography{neurips_2026}

@article{UCI,
author = {Panzarasa, Pietro and Opsahl, Tore and Carley, Kathleen M.},
title = {Patterns and dynamics of users' behavior and interaction: Network analysis of an online community},
year = {2009},
volume = {60},
number = {5},
issn = {1532-2882},
journal = {J. Am. Soc. Inf. Sci. Technol.},
pages = {911–932},
}

@inproceedings{OTC,
author = {Kumar, Srijan and Hooi, Bryan and Makhija, Disha and Kumar, Mohit and Faloutsos, Christos and Subrahmanian, V.S.},
title = {REV2: Fraudulent User Prediction in Rating Platforms},
year = {2018},
booktitle = {Proceedings of the Eleventh ACM International Conference on Web Search and Data Mining},
pages = {333–341},
}

@inproceedings{Social_Network,
author = {Tang, Jie and Zhang, Jing and Yao, Limin and Li, Juanzi and Zhang, Li and Su, Zhong},
title = {ArnetMiner: extraction and mining of academic social networks},
year = {2008},
booktitle = {Proceedings of the 14th ACM SIGKDD International Conference on Knowledge Discovery and Data Mining},
pages = {990–998},
series = {KDD '08}
}

@inproceedings{Reddit,
author = {Kumar, Srijan and Hamilton, William L. and Leskovec, Jure and Jurafsky, Dan},
title = {Community Interaction and Conflict on the Web},
year = {2018},
pages = {933–943},
series = {WWW '18}
}

@inproceedings{as-733,
author = {Leskovec, Jure and Kleinberg, Jon and Faloutsos, Christos},
title = {Graphs over time: densification laws, shrinking diameters and possible explanations},
year = {2005},
isbn = {159593135X},
publisher = {Association for Computing Machinery},
address = {New York, NY, USA},
url = {https://doi.org/10.1145/1081870.1081893},
doi = {10.1145/1081870.1081893},
booktitle = {Proceedings of the Eleventh ACM SIGKDD International Conference on Knowledge Discovery in Data Mining},
pages = {177–187},
numpages = {11},
location = {Chicago, Illinois, USA},
series = {KDD '05}
}

@article{Aligraph,
author = {Zhu, Rong and Zhao, Kun and Yang, Hongxia and Lin, Wei and Zhou, Chang and Ai, Baole and Li, Yong and Zhou, Jingren},
title = {AliGraph: a comprehensive graph neural network platform},
year = {2019},
issue_date = {August 2019},
volume = {12},
number = {12},
issn = {2150-8097},
journal = {Proc. VLDB Endow.},
pages = {2094–2105},
}

@inproceedings{ESDG,
author = {Chakaravarthy, Venkatesan T. and Pandian, Shivmaran S. and Raje, Saurabh and Sabharwal, Yogish and Suzumura, Toyotaro and Ubaru, Shashanka},
title = {Efficient scaling of dynamic graph neural networks},
year = {2021},
booktitle = {Proceedings of the International Conference for High Performance Computing, Networking, Storage and Analysis},
articleno = {77},
series = {SC '21}
}

@inproceedings{BLAD,
author = {Fu, Kaihua and Chen, Quan and Yang, Yuzhuo and Shi, Jiuchen and Li, Chao and Guo, Minyi},
title = {BLAD: Adaptive Load Balanced Scheduling and Operator Overlap Pipeline For Accelerating The Dynamic GNN Training},
year = {2023},
booktitle = {Proceedings of the International Conference for High Performance Computing, Networking, Storage and Analysis},
articleno = {37},
series = {SC '23}
}

@inproceedings{FedSage,
title={Subgraph Federated Learning with Missing Neighbor Generation},
author={Ke ZHANG and Carl Yang and Xiaoxiao Li and Lichao Sun and Siu Ming Yiu},
booktitle={Advances in Neural Information Processing Systems},
editor={A. Beygelzimer and Y. Dauphin and P. Liang and J. Wortman Vaughan},
year={2021},
}

@article{FedEgo,
author = {Zhang, Taolin and Mai, Chengyuan and Chang, Yaomin and Chen, Chuan and Shu, Lin and Zheng, Zibin},
title = {FedEgo: Privacy-preserving Personalized Federated Graph Learning with Ego-graphs},
year = {2023},
issue_date = {February 2024},
volume = {18},
number = {2},
issn = {1556-4681},
journal = {ACM Trans. Knowl. Discov. Data},
month = nov,
articleno = {40},
}

@article{FedGCN,
  title={FedGCN: convergence-communication tradeoffs in federated training of graph convolutional networks},
  author={Yao, Yuhang and Jin, Weizhao and Ravi, Srivatsan and Joe-Wong, Carlee},
  journal={Advances in neural information processing systems},
  volume={36},
  pages={79748--79760},
  year={2023}
}

@InProceedings{FedDGL,
  title = 	 {{FedDGL}: {F}ederated Dynamic Graph Learning for Temporal Evolution and Data Heterogeneity},
  author =       {Xie, Zaipeng and Likun, Li and Chen, Xiangbin and Yu, Hao and Huang, Qian},
  booktitle = 	 {Proceedings of the 16th Asian Conference on Machine Learning},
  pages = 	 {463--478},
  year = 	 {2025},
  volume = 	 {260},
  month = 	 {05--08 Dec},
}

@inproceedings{ROLAND,
  title={ROLAND: graph learning framework for dynamic graphs},
  author={You, Jiaxuan and Du, Tianyu and Leskovec, Jure},
  booktitle={Proceedings of the 28th ACM SIGKDD Conference on Knowledge Discovery and Data Mining},
  pages={2358--2366},
  year={2022}
}

@inproceedings{EvolveGCN,
  title={Evolvegcn: Evolving graph convolutional networks for dynamic graphs},
  author={Pareja, Aldo and Domeniconi, Giacomo and Chen, Jie and Ma, Tengfei and Suzumura, Toyotaro and Kanezashi, Hiroki and Kaler, Tim and Schardl, Tao and Leiserson, Charles},
  booktitle={Proceedings of the AAAI conference on artificial intelligence},
  volume={34},
  number={04},
  pages={5363--5370},
  year={2020}
}

@inproceedings{DMGCN,
  title={Dynamic and multi-channel graph convolutional networks for aspect-based sentiment analysis},
  author={Pang, Shiguan and Xue, Yun and Yan, Zehao and Huang, Weihao and Feng, Jinhui},
  booktitle={Findings of the association for computational linguistics: ACL-IJCNLP 2021},
  pages={2627--2636},
  year={2021}
}

@article{DyGFormer,
  title={Towards better dynamic graph learning: New architecture and unified library},
  author={Yu, Le and Sun, Leilei and Du, Bowen and Lv, Weifeng},
  journal={Advances in Neural Information Processing Systems},
  volume={36},
  pages={67686--67700},
  year={2023}
}

@inproceedings{DyGRAIN,
  title={DyGRAIN: An Incremental Learning Framework for Dynamic Graphs.},
  author={Kim, Seoyoon and Yun, Seongjun and Kang, Jaewoo},
  booktitle={IJCAI},
  pages={3157--3163},
  year={2022}
}

@article{fedgraph,
  title={Fedgraph: Federated graph learning with intelligent sampling},
  author={Chen, Fahao and Li, Peng and Miyazaki, Toshiaki and Wu, Celimuge},
  journal={IEEE Transactions on Parallel and Distributed Systems},
  volume={33},
  number={8},
  pages={1775--1786},
  year={2021},
  publisher={IEEE}
}

@article{Louvain,
  title={Fast unfolding of communities in large networks},
  author={Blondel, Vincent D and Guillaume, Jean-Loup and Lambiotte, Renaud and Lefebvre, Etienne},
  journal={Journal of statistical mechanics: theory and experiment},
  volume={2008},
  number={10},
  pages={P10008},
  year={2008},
  publisher={IOP Publishing}
}

@inproceedings{forgetting,
  title={Continual learning with tiny episodic memories},
  author={Chaudhry, Arslan and Rohrbach, Marcus and Elhoseiny, Mohamed and Ajanthan, Thalaiyasingam and Dokania, P and Torr, P and Ranzato, M},
  booktitle={Workshop on Multi-Task and Lifelong Reinforcement Learning},
  year={2019}
}

@inproceedings{ER-GNN,
  title={Overcoming catastrophic forgetting in graph neural networks with experience replay},
  author={Zhou, Fan and Cao, Chengtai},
  booktitle={Proceedings of the AAAI conference on artificial intelligence},
  volume={35},
  number={5},
  pages={4714--4722},
  year={2021}
}

@inproceedings{TWP,
  title={Overcoming catastrophic forgetting in graph neural networks},
  author={Liu, Huihui and Yang, Yiding and Wang, Xinchao},
  booktitle={Proceedings of the AAAI conference on artificial intelligence},
  volume={35},
  number={10},
  pages={8653--8661},
  year={2021}
}

@inproceedings{FederatedScope,
author = {Wang, Zhen and Kuang, Weirui and Xie, Yuexiang and Yao, Liuyi and Li, Yaliang and Ding, Bolin and Zhou, Jingren},
title = {FederatedScope-GNN: Towards a Unified, Comprehensive and Efficient Package for Federated Graph Learning},
year = {2022},
isbn = {9781450393850},
publisher = {Association for Computing Machinery},
address = {New York, NY, USA},
url = {https://doi.org/10.1145/3534678.3539112},
doi = {10.1145/3534678.3539112},
booktitle = {Proceedings of the 28th ACM SIGKDD Conference on Knowledge Discovery and Data Mining},
pages = {4110–4120},
numpages = {11},
location = {Washington DC, USA},
series = {KDD '22}
}

@inproceedings{llcg,
    author = {Ramezani, Morteza and Cong, Weilin and Mahdavi, Mehrdad and Kandemir, Mahmut T and Sivasubramaniam, Anand},
    title = {Learn locally, correct globally: A distributed algorithm for training graph neural networks},
    booktitle = {10th International Conference on Learning Representations, ICLR},
    year = {2022}
}

@article{temporalpart,
  title={Practical and high-quality partitioning algorithm for large-scale and time-evolving graphs},
  author={Luo, Xiangyu and Luo, Yingxiao and Xin, Gang and Gui, Xiaolin and Wang, Jia and Guo, Cheng},
  journal={Knowledge-Based Systems},
  volume={227},
  pages={107211},
  year={2021},
  publisher={Elsevier}
}

@article{applications,
  title={Recent Advances in Efficient Dynamic Graph Processing},
  author={Chen, Zi and Liang, Keke and Yuan, Long and Zhang, Wenjie and Yang, Zhengyi},
  journal={Applied Sciences},
  volume={15},
  number={11},
  pages={6003},
  year={2025},
  publisher={MDPI}
}

@InProceedings{ngil,
  title = 	 {Towards Robust Graph Incremental Learning on Evolving Graphs},
  author =       {Su, Junwei and Zou, Difan and Zhang, Zijun and Wu, Chuan},
  booktitle = 	 {Proceedings of the 40th International Conference on Machine Learning},
  pages = 	 {32728--32748},
  year = 	 {2023},
  editor = 	 {Krause, Andreas and Brunskill, Emma and Cho, Kyunghyun and Engelhardt, Barbara and Sabato, Sivan and Scarlett, Jonathan},
  volume = 	 {202},
  series = 	 {Proceedings of Machine Learning Research},
  month = 	 {23--29 Jul},
  publisher =    {PMLR},
  url = 	 {https://proceedings.mlr.press/v202/su23a.html},
}

@article{FedGNNLDP,
  title={FedGNNLDP: Federated Graph Neural Network with Locally Differential Privacy},
  author={Liu, Yaqi and Zhang, Yue and He, Pinzhen and Fang, Shuzhen},
  journal={Computers \& Security},
  pages={104757},
  year={2025},
  publisher={Elsevier}
}

@inproceedings{MIA1,
  title={Membership inference attack on graph neural networks},
  author={Olatunji, Iyiola E and Nejdl, Wolfgang and Khosla, Megha},
  booktitle={2021 Third IEEE International Conference on Trust, Privacy and Security in Intelligent Systems and Applications (TPS-ISA)},
  pages={11--20},
  year={2021},
  organization={IEEE}
}

@article{MIA2,
  title={Graph-level label-only membership inference attack against graph neural networks},
  author={Dai, Jiazhu and Lu, Yubing},
  journal={Applied Sciences},
  volume={15},
  number={9},
  pages={5086},
  year={2025},
  publisher={MDPI}
}

@article{MIA3,
  title={Subgraph structure membership inference attacks against graph neural networks},
  author={Wang, Xiuling and Wang, Wendy Hui},
  journal={Proceedings on Privacy Enhancing Technologies},
  year={2024}
}

@article{MIA4,
  title={Attention-based membership inference attacks on graph neural network through topological features},
  author={Guan, Faqian and Zhu, Tianqing and Tong, Hanjin and Zhou, Wanlei},
  journal={IEEE Transactions on Dependable and Secure Computing},
  year={2025},
  publisher={IEEE}
}

@inproceedings{GIA1,
  title={$\{$SoK$\}$: Gradient Inversion Attacks in Federated Learning},
  author={Carletti, Vincenzo and Foggia, Pasquale and Mazzocca, Carlo and Parrella, Giuseppe and Vento, Mario},
  booktitle={34th USENIX Security Symposium (USENIX Security 25)},
  pages={6439--6459},
  year={2025}
}

@article{GIA2,
  title={Deep leakage from gradients},
  author={Zhu, Ligeng and Liu, Zhijian and Han, Song},
  journal={Advances in neural information processing systems},
  volume={32},
  year={2019}
}

@inproceedings{GIA3,
  title={Gradients Stand-in for Defending Deep Leakage in Federated Learning},
  author={Hu, Yi and Ren, Hanchi and Hu, Chen and Li, Yiming and Deng, Jingjing and Xie, Xianghua},
  booktitle={2024 International Conference on Computing in Natural Sciences, Biomedicine and Engineering (COMCONF)},
  pages={53--64},
  year={2024},
  organization={IEEE}
}

@article{GUN1,
  title={A survey on machine unlearning: Techniques and new emerged privacy risks},
  author={Liu, Hengzhu and Xiong, Ping and Zhu, Tianqing and Yu, Philip S},
  journal={Journal of Information Security and Applications},
  volume={90},
  pages={104010},
  year={2025},
  publisher={Elsevier}
}

@inproceedings{GUN2,
  title={Machine Unlearning Fails to Remove Data Poisoning Attacks},
  author={Martin Pawelczyk and Jimmy Z. Di and Yiwei Lu and Gautam Kamath and Ayush Sekhari and Seth Neel},
  booktitle={The Thirteenth International Conference on Learning Representations},
  year={2026}
}

@inproceedings{TGAT,
  title={TGAT: temporal graph attention network for blockchain phishing scams detection},
  author={Dai, Chaofan and Tang, Qideng and Ding, Huahua},
  booktitle={2024 International Conference on Computer, Information and Telecommunication Systems (CITS)},
  pages={1--7},
  year={2024},
  organization={IEEE}
}

@article{TGN,
  title={Temporal graph networks for deep learning on dynamic graphs},
  author={Rossi, Emanuele and Chamberlain, Ben and Frasca, Fabrizio and Eynard, Davide and Monti, Federico and Bronstein, Michael},
  journal={arXiv preprint arXiv:2006.10637},
  year={2020}
}

@article{DPDGL,
  title={Personalized graph federated learning with differential privacy},
  author={Gauthier, Francois and Gogineni, Vinay Chakravarthi and Werner, Stefan and Huang, Yih-Fang and Kuh, Anthony},
  journal={IEEE Transactions on Signal and Information Processing over Networks},
  volume={9},
  pages={736--749},
  year={2023},
  publisher={IEEE}
}

@article{huang2024tgb2,
  title={TGB 2.0: A Benchmark for Learning on Temporal Knowledge Graphs and Heterogeneous Graphs},
  author={Gastinger, Julia and Huang, Shenyang and Galkin, Mikhail and Loghmani, Erfan and Parviz, Ali and Poursafaei, Farimah and Danovitch, Jacob and Rossi, Emanuele and Koutis, Ioannis and Stuckenschmidt, Heiner and      Rabbany, Reihaneh and Rabusseau, Guillaume},
  journal={Advances in Neural Information Processing Systems},
  year={2024}
}







\appendix

\section{Additional Details of CoLearnPartition} \label{appen:gpa}

\subsection{Overall Flow of CoLearnPartition}
Figure \ref{fig:overallgpa} shows the complete procedure of CoLearnPartition, where the number of partitions depends on the size of the incoming graph. Therefore, as the graph size grows, the number of partitions required will also increase, and the previous partitioning will no longer be \textit{valid}. However, if the number of partitions in the previous snapshot is the same as the current snapshot, we can reuse the previous partition so that the existing graph elements do not need to be reprocessed.
\begin{figure*}[t]
  \centering
  \includegraphics[width=1.0\textwidth]{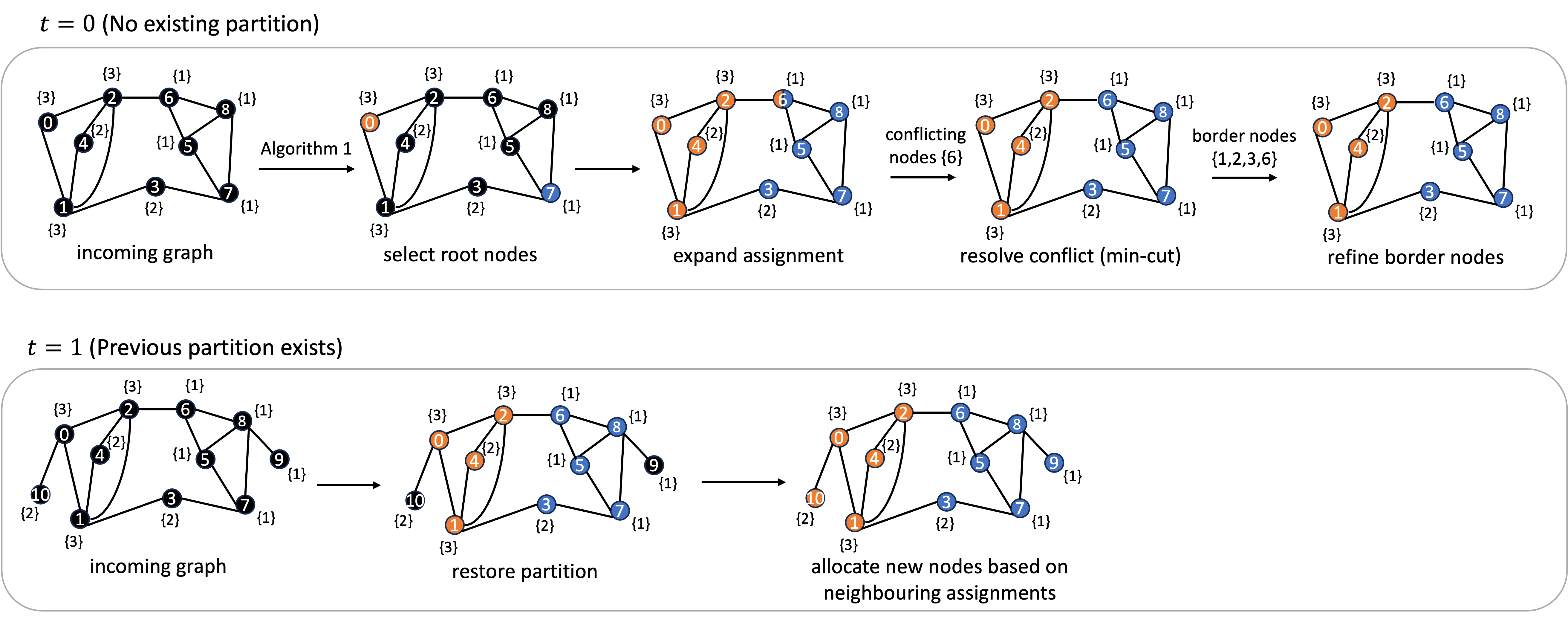}
  \caption{Overall Procedure of CoLearnPartition splitting a graph into 2 clients. Top part illustrates the case when there is no previous partition. Bottom part illustrates the incremental partitioning process. Each node is annotated with its node label (shown in curly brackets)}
  \label{fig:overallgpa}
\end{figure*}

\subsection{Formulae of Scoring Function}
The primary node allocation scoring function $P(v,S_i)$ for resolving conflicts via minimum cut edges is defined as:
\begin{equation*}
P(v, S_i) = |\mathcal{N}(v) \cap V_i|
\end{equation*}

where:
\begin{itemize}
\item $v$ is the node being allocated
\item $S_i = (V_i, E_i)$ is candidate subgraph $i$
\item $\mathcal{N}(v)$ is the neighbour set of $v$
\item $V_i$ is the node set of subgraph $S_i$
\end{itemize}

The allocation decision selects the subgraph $S^*$ that maximizes this score:

\begin{equation}
S^* = \underset{S_i \in \mathcal{G}}{\arg\max} P(v, S_i)
\end{equation}

where $\mathcal{G}$ is the set of candidate subgraphs.

The refinement scoring function evaluates potential moves of a node $v$ to neighbouring subgraphs based on a weighted combination of balance and label diversity metrics. For each candidate subgraph $S_j \in \mathcal{G}$, the score is computed as:

\begin{equation*}
R(v, S_j) = \alpha \cdot \text{balance}(v, S_j) + (1-\alpha) \cdot \text{label\_div}(v, S_j)
\end{equation*}

where $\alpha = 0.5$ balances the two objectives. The balance score $\text{balance}(v, S_j)$ considers:
\begin{itemize}
    \item Current subgraph sizes and overall graph balance
    \item Synthetic edge connections (if present)
    \item Isolation status of nodes
\end{itemize}

The label diversity metric is computed as:
\begin{equation}
\text{label\_div}(v, S_j) = 1 - \frac{|\{u \in S_j : L(u) = L(v)\}|}{|S_j| + \epsilon}
\end{equation}

where $L(u)$ denotes the label of node $u$ and $\epsilon$ prevents division by zero. The refinement moves $v$ to the subgraph $R^*$ that maximizes:

\begin{equation}
R^* = \underset{S_j \in \mathcal{G}}{\arg\max} [R(v, S_j) - R(v, S^*)] > \tau
\end{equation}

where $\tau$ is a minimum improvement threshold. The algorithm only moves $v$ if the improvement exceeds both $\tau$ and any previously seen improvement during the evaluation.

\subsection{Time Complexity of CoLearnPartition}
CoLearnPartition employs a multi-stage strategy including graph preprocessing to partition dynamic graphs while minimising cross-client edges and preserving edge balance. We analyse its complexity step-wise:

\textbf{Phase 1: Preprocessing.} Using breadth-first search (BFS) to visit connected components requires $O(|V|+|E|)$ time for a graph $G=(V,E)$. Connecting $m$ connected components costs $O(mlogm)$. Here, $O(|V|+|E|)$ dominates Phase 1. 

\textbf{Phase 2: Primary Stage of Conflict Resolving.} Computing $k$ furthest nodes via BFS requires $O(k(|V|+|E|))$. Resolving the number of cut edges involves identifying neighbours of each conflicting node and counting subgraph interaction. Consider a cross-client node $v$, we iterate through all $d_v$ neighbours, where $d_v$ is the degree of node $v$. For each neighbour, we use constant time to lookup and record the partition it is in. Afterwards, we find the partition with the highest count using $O(k)$ time where $k$ is the number of subgraphs. Therefore, the case complexity for one conflicting node is $O(k + d_v)$. For $c$ conflicting nodes, the time complexity becomes $O(c \cdot (k+d_v))$. The cost is dominated by neighbour iteration $O(d_v)$, in sparse graphs such as social networks, we have $d_v \ll |V|$.

\textbf{Phase 3: Refinement Stage for conflicting and border nodes.} This stage evaluates label distribution and edge balance to reassign border nodes. In terms of node label scoring, we iterate through each subgraph's nodes (i.e. $O(|S_i|)$) and compare the frequencies in constant time, yielding the worst time complexity of $O(|V|)$. For balance scoring, checking the size of each subgraph also takes $O(|S_i|)$. Therefore, Phase 3 yields a time complexity of $O(|V|)$.

\textbf{Overall Complexity.} Combining the three phases, we have:
\begin{equation*}
    \mathcal{T}_{\text{total}} = \underbrace{O(|V| + |E|)}_{\text{Preprocessing}} + \underbrace{O\left(k(|V| + |E|) + d_v\right)}_{\text{BFS + Primary Resolve}} + \underbrace{O(|V|)}_{\text{Refinement}}
\end{equation*}

Overall speaking, the full procedure has a time complexity of $O(|V|+|E|)$. 

\subsection{Pseudocode of CoPartition}
\begin{algorithm}
\caption{CoLearnPartition (without existing valid partition)}
\label{alg:colearnpartition}
\begin{algorithmic}[1]
\Require Snapshot $G_t = (V_t, E_t)$

\State Select $M$ seed nodes that are maximally distant
\State Grow provisional partitions via BFS from each seed
\State Identify conflicting nodes reached by multiple seeds

\State \textbf{/* Primary stage: hard constraint */}
\State Assign conflicting nodes to minimize cross-client edges

\State Identify border nodes with neighbours in other partitions

\State \textbf{/* Secondary stage: soft refinement */}
\State Refine border-node assignments to balance partition edge volume and node label distribution while limiting new edge cuts

\State \textbf{Return:} Updated partitions $\{V_t^{(1)}, \dots, V_t^{(M)}\}$
\end{algorithmic}
\end{algorithm}

\section{Experiments on various Partitioning Stage Configurations} \label{appen:gpa_experiments}
In terms of primary and secondary stage of conflict resolving, there are three combinations to be explored:
\begin{enumerate}
    \item \textbf{Primary:} Minimize Cut Edges; \textbf{Secondary:} Edge Size Balance and Node Label Diversity [\texttt{MinCut\_First}]
    \item \textbf{Primary:} Edge Size Balance; \textbf{Secondary:} Minimize Cut Edges and Node Label Diversity [\texttt{Balance\_First}]
    \item \textbf{Primary:} Node Label Diversity; \textbf{Secondary:} Minimize Cut Edges and Edge Size Balance [\texttt{Label\_First}]
\end{enumerate}
To formally evaluate the impact of each configuration, we used three numerical metrics to measure each objective:

To measure the impact of \textit{minimizing cut edges}, we use the \textbf{cut edge ratio}, which is the proportion of cut edges in the graph:
\begin{equation*}
    Cut\_Edge\_Ratio = \frac{\text{Number of Cut Edges}}{\text{Total Number of Edges in Graph}}
\end{equation*}

To measure the direct impact of cut edges to the graph learning pipeline, we also recorded the time of node embedding exchange.

To quantify \textit{how balanced the subgraphs} are, we measure the \textbf{coefficient of variation (CoV)}, which captures the relative variation in the subgraphs' edge counts:
\begin{equation*}
    CoV = \frac{\text{Standard Deviation } \sigma}{\text{Mean } \mu}
\end{equation*}
CoV values are usually bounded between 0 and 1, a CoV close to 0 means more consistency in the subgraph sizes (i.e. balanced subgraph size), while a high CoV means the subgraph sizes are more varied. In rare cases, CoV values can exceed 1, indicating severe imbalance. 

For \textit{node label diversity}, we measure with \textbf{shannon entropy} $H$, which is usually used for categorical distributions. It effectively measures how evenly distributed the node labels are. Given $n$ label classes $p_1, p_2, ..., p_n$, shannon entropy is computed as:
\begin{equation*}
    H = - \sum_{i=1}^n p_i \cdot log_2(p_i)
\end{equation*}
While the value is bounded between 0 and $log_2(n)$, a high entropy close to $log_2(n)$ implies a more uniform label distribution, while a lower entropy close to 0 indicates a skewed distribution. 

The direct impact of graph balance on graph learning can be reflected by the variance of clients' local accuracies. 

\subsection{Cut Edge Ratio} 
\begin{figure*}[t]
  \centering
  \includegraphics[width=1.0\textwidth]{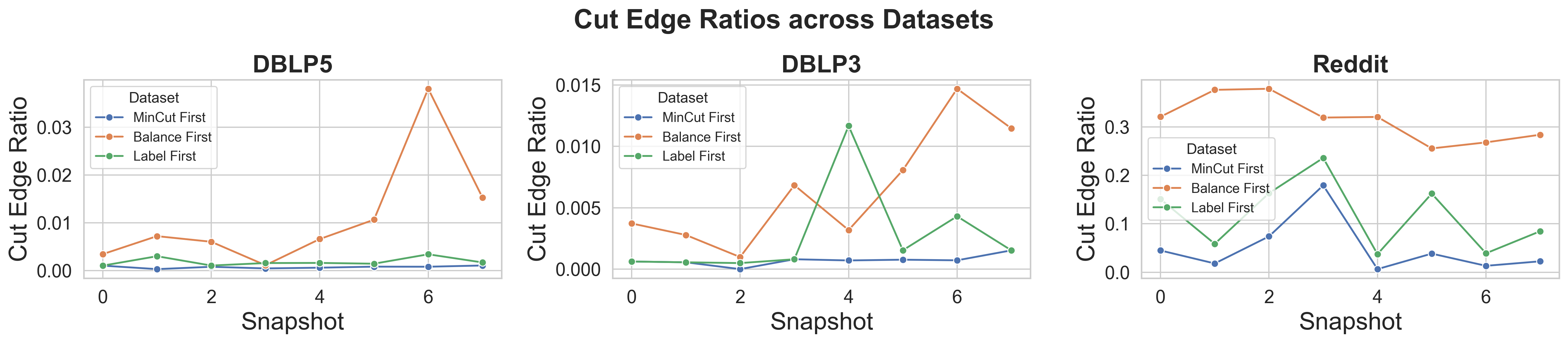}
  \caption{Cut Edge Ratio across configurations for each dataset}
  \label{fig:cutedgeratio}
\end{figure*}
Figure \ref{fig:cutedgeratio} demonstrates that \texttt{MinCut\_First} consistently achieves the lowest cut edge ratio in all three datasets, validating its effectiveness in minimising edge cuts. In contrast, \texttt{Label\_First} exhibits a 2.34× higher cut edge ratio, while \texttt{Balance\_First} performs significantly worse, with ratios up to 7.75× higher than \texttt{MinCut\_First}. This difference highlights the inherent trade-off between balanced partitions (\texttt{Balance\_First}) and minimal edge cuts (\texttt{MinCut\_First}), with \texttt{Label\_First} offering an intermediate compromise.

\subsection{Time Consumption on Node Embedding Exchange}
\begin{figure*}[t]
  \centering
  \includegraphics[width=1.0\textwidth]{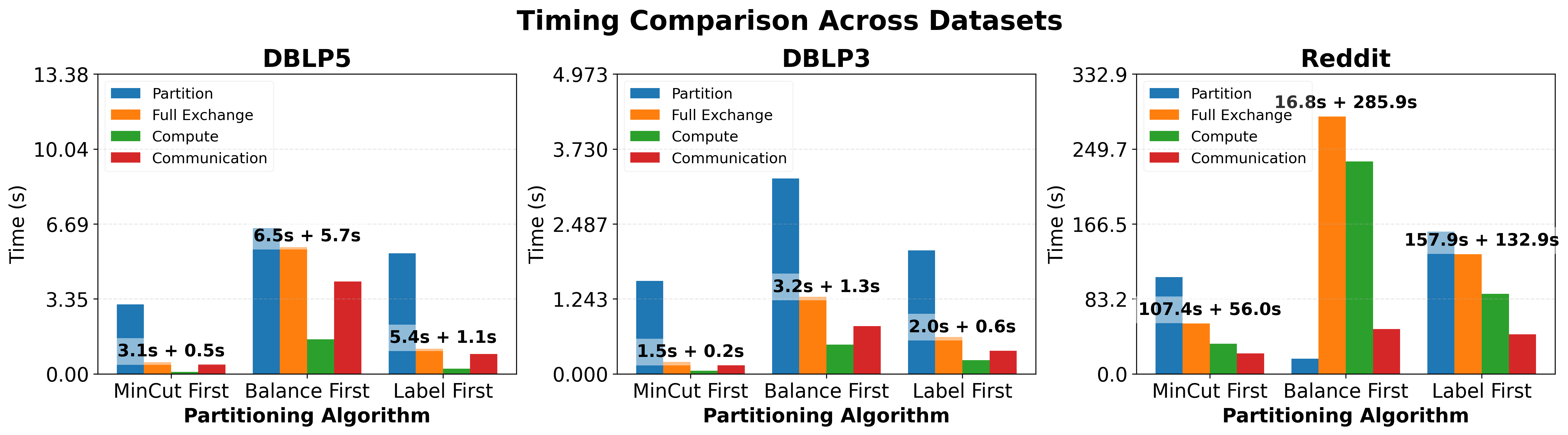}
  \caption{Time taken for Graph Partition and Node Embedding Exchange across configurations for each dataset (Each bolded label value indicates the partition time, (i.e., blue bar) + total time to perform node embedding exchange, (i.e., orange bar))}
  \label{fig:gpa_strat_time}
\end{figure*}
To further assess the impact of cut edges on the graph learning pipeline, we analysed the node embedding exchange time in Figure \ref{fig:gpa_strat_time}. The total exchange time (orange bar) comprises both communication time (red bar) and computation time (green bar). Together with the graph partition time, their exact values are annotated above each bar. 

The results reveal that \texttt{Balance\_First} incurs significantly longer embedding exchange times compared \texttt{MinCut\_First} and \texttt{Label\_First}, directly correlating with its higher cut edges count. Notably, \texttt{MinCut\_First} with the lowest cut, achieves the fastest exchange times. Using \texttt{MinCut\_First} as the baseline, \texttt{Balance\_First} requires up to 5.11× more time, while \texttt{Label\_First} takes up to 3.0× longer, demonstrating the computational cost of excessive cross-client edges. Given node embedding exchange is a critical phase in our framework, optimizing this procedure through efficient configuration is essential. Therefore, we exclude \texttt{Balance\_First} as a viable candidate due to its prohibitive overhead.

\subsection{Coefficient of Variations (CoV)}
\begin{figure*}[t]
  \centering
  \includegraphics[width=1.0\textwidth]{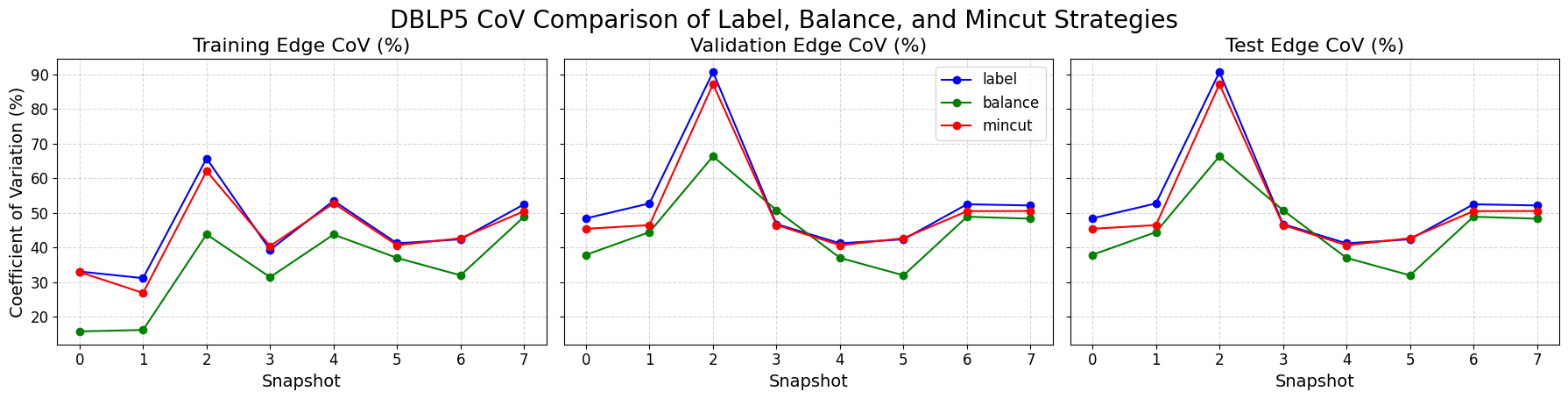}
  \caption{Coefficient of Variation of training, validation and test edges across configurations (DBLP5)}
  \label{fig:dblp5_cov}
\end{figure*}
\begin{figure*}[t]
  \centering
  \includegraphics[width=1.0\textwidth]{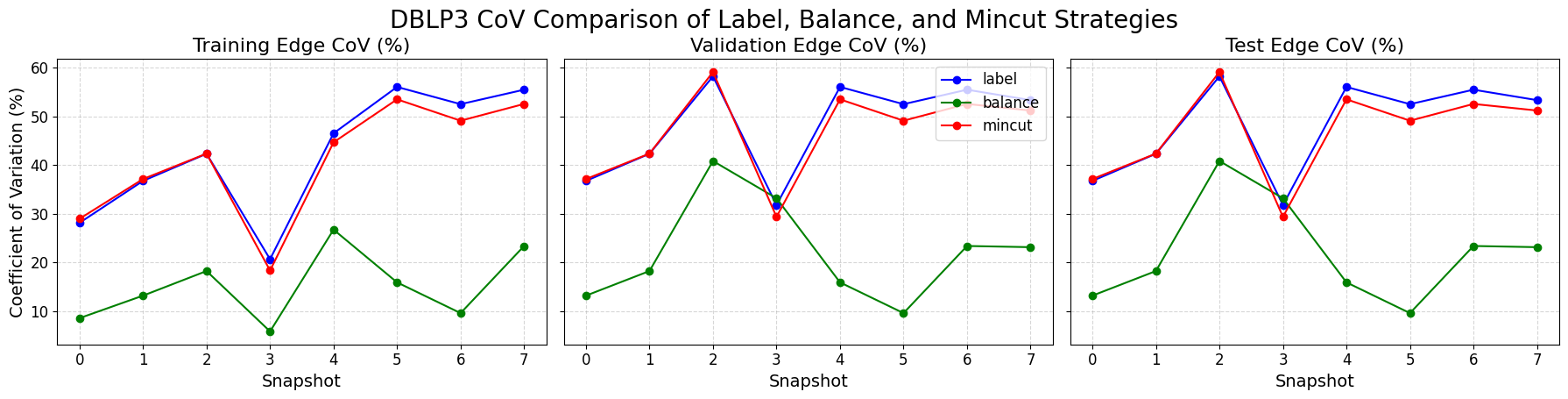}
  \caption{Coefficient of Variation of training, validation and test edges across configurations (DBLP3)}
  \label{fig:dblp3_cov}
\end{figure*}
\begin{figure*}[t]
  \centering
  \includegraphics[width=1.0\textwidth]{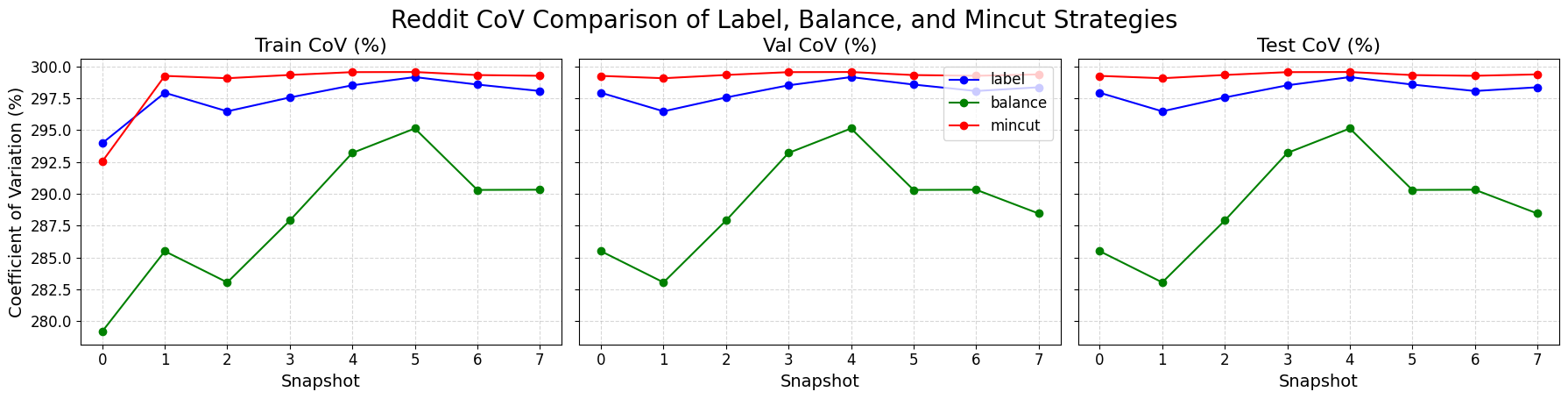}
  \caption{Coefficient of Variation of training, validation and test edges across configurations (Reddit)}
  \label{fig:reddit_cov}
\end{figure*}
We next analyse the impact of partition balancing for each dataset. Figures \ref{fig:dblp5_cov},\ref{fig:dblp3_cov}, and \ref{fig:reddit_cov} present the coefficient of variation for training, validation, and test edge distributions for three different datasets respectively, where a perfect balance corresponds to CoV = 0. Surprisingly, \texttt{MinCut\_First} and \texttt{Label\_First} occasionally achieve lower CoV values than \texttt{Balance\_First} in DBLP5 and DBLP3, suggesting that explicit balancing does not always yield the most uniform splits. Secondly, all configurations exhibit substantially higher CoV values in Reddit (i.e., higher than 1), indicating severe partition imbalance. We attribute these to fundamental graph structural properties: real-world graphs exhibit heavy-tailed degree distributions. This forces compromises, for instance, balancing partitions would require excessive edge cuts, while avoiding cuts leads to concentration in single partitions. Our results suggest this structural bias dominates even explicit balancing efforts.

\subsection{Shannon Entropy}
\begin{figure*}[t]
  \centering
  \includegraphics[width=1.0\textwidth]{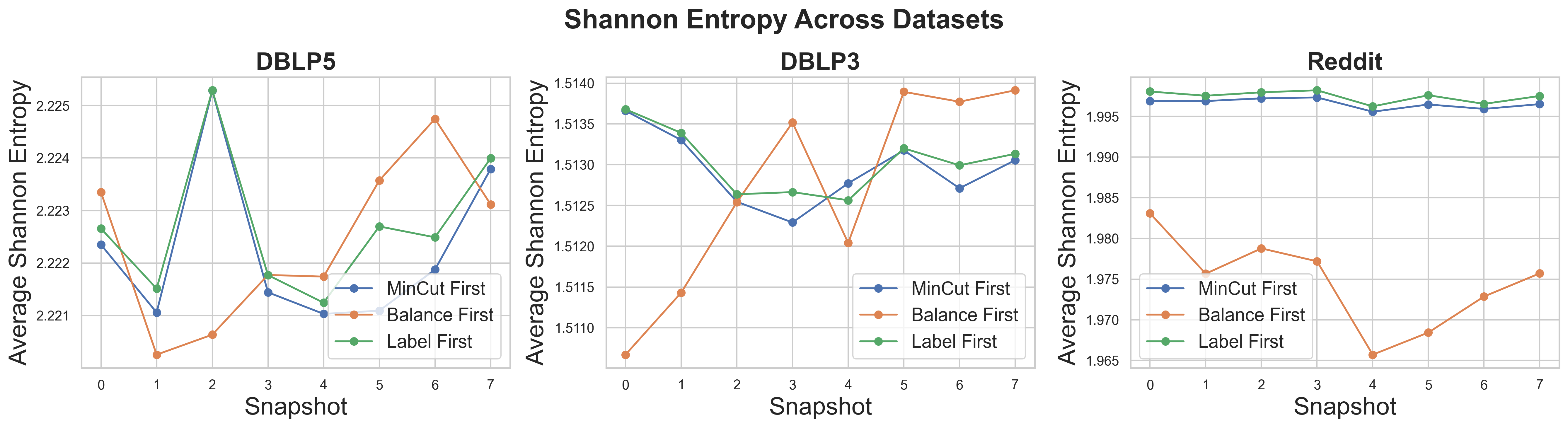}
  \caption{Shannon Entropy across configurations for each dataset}
  \label{fig:shannon}
\end{figure*}
Shannon Entropy exhibits more uniform distribution if its value is close to $log_2(n)$, where $n$ is the number of classes of the dataset. Table \ref{tab:best_shannon} shows the best value for each dataset.
\begin{table}
    \centering
    \begin{tabular}{c|c|c}
         \textbf{Dataset}&  \textbf{Number of classes $n$} &  \textbf{Best Value ($log_2(n)$)}\\
         \midrule
         DBLP5&  5&  2.32\\
         DBLP3&  3&  1.58\\
         Reddit&  4&  2.00\\
    \end{tabular}
    \caption{Best Shannon Entropy value for each dataset}
    \label{tab:best_shannon}
\end{table}
Figure \ref{fig:shannon} reveals two key findings about label diversity. Firstly \texttt{MinCut\_First} and \texttt{Label\_First} achieve nearly identical entropy values, consistently approaching near-optimal entropy for each dataset. Their convergence suggests that \texttt{MinCut\_First} indirectly preserves label diversity as effectively as primary label optimization in \texttt{Label\_First}. In contrast, \texttt{Balance\_First} exhibits more fluctuated and reduced entropy. This reflects inherent conflicts between its primary partition balancing objective and label distribution preservation, suggesting force graph splitting to achieve size equality disrupts natural communities sharing the same label. 

\subsection{Local Client Accuracies}
\begin{figure*}[t]
  \centering
  \includegraphics[width=1.0\textwidth]{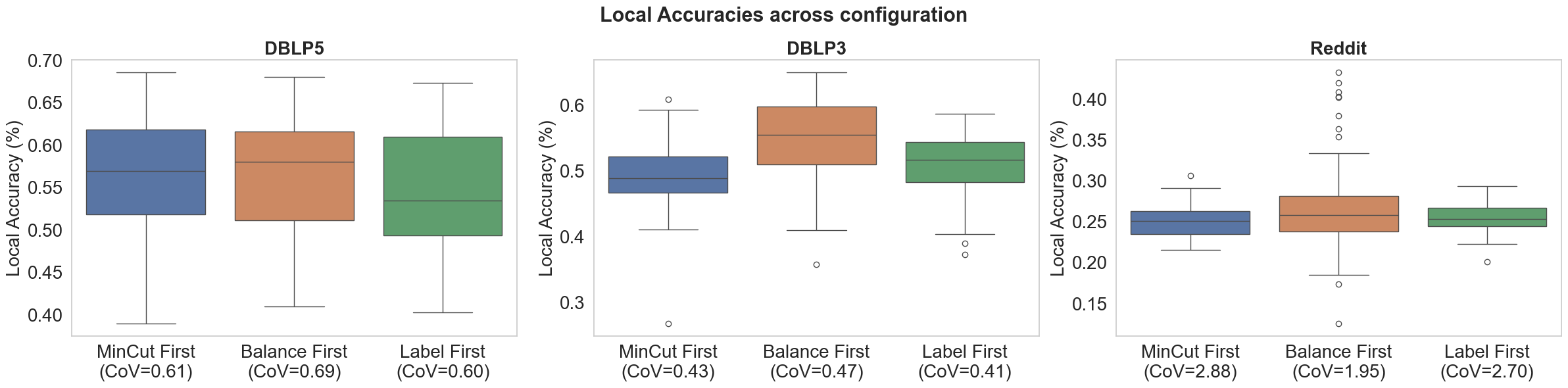}
  \caption{Local Accuracies across configurations for each dataset (The average of the CoV values across snapshots are computed and listed for reference.)}
  \label{fig:cov_vs_acc}
\end{figure*}
We evaluate the impact of strategies on local client performances through boxplots analysis in Figure \ref{fig:cov_vs_acc}. The horizontal line within the box marks the median accuracy (50th percentile), providing the \textbf{central tendency of clients}. The box boundaries indicate the 25th (Q1) and 75th (Q3) percentiles, showing where the middle 50\% of client accuracies fall. The interquartile ranges (IQR) width reflects \textbf{performance consistency} - narrower boxes denote more uniform client experiences. Points beyond 1.5×IQR indicates statistical \textbf{outliers}. These reveal edge cases where clients significantly over- or under-perform relative to the group. 

\texttt{MinCut\_First} and \texttt{Label\_First} demonstrate comparable results IQR widths, indicating similar training stability. This corresponds to the above observations, where they achieve similar cut edge ratio, CoV and Shannon Entropy. On the other hand, despite prioritizing partition balancing, \texttt{Balance\_First} exhibits wider IQRs than \texttt{MinCut\_First} and \texttt{Label\_First}, with more outliers  perform unusually well or poorly in Reddit dataset.
This contradicts to the expectation that balanced partitions should yield more consistent client performance. 

\texttt{Balance\_First}’s strict size constraints disrupt natural label communities, evidenced by lower Shannon entropy in Figure \ref{fig:shannon}, creating unnatural subgraphs where clients with coherent local data overfit and disrupted data underfit. In contrast, \texttt{MinCut\_First} and \texttt{Label\_First} preserve graph homophily, yielding more predictable accuracies. This proves that partition balancing does not guarantee stable and consistent local client performances.

\subsection{Overall Time Taken} 
Since \texttt{MinCut\_First} and \texttt{Label\_First} remain to achieve similar performance in node embedding exchange time consumption and local accuracies, we analyse their total time taken, including graph partitioning and node embedding exchange. Figure \ref{fig:gpa_strat_time} shows that in all three datasets, \texttt{Label\_First} requires 1.81× longer time to complete the procedures. This is because of complex label-aware partitioning and higher cut edge ratios. In label aware partitioning, additional computation is needed to track and optimize label distribution. As \texttt{Label\_First} yield higher cut edge ratio (2.34× higher than \texttt{MinCut\_First}), it also increases communication overhead in node embedding exchange scheme.

\subsection{Conclusion} 
In conclusion, \texttt{MinCut\_First} emerges as the optimal choice, offering:
\begin{itemize}
    \item \textbf{Time Efficiency:} Fastest execution on partitioning and training
    \item \textbf{Reliable Performance:} Tightest IQRs and better accuracies
    \item \textbf{Theoretical Alignment:} Minimal edge cuts preserve graph properties is crucial for GNN training
\end{itemize}

\section{Details of Lightweight Dynamic GNN Learning} \label{appen:learning}

\subsubsection{Step 1: $k$-hop subgraph expansion} \label{sesh:learning1}
Consider the delta graph $\Delta G_{(t-1,t)}$, for a node $v \in \Delta V_{(t-1,t)}$ that is updated in snapshot $t$, all nodes within $k$-hop away should also be updated. Therefore, the client first performs a $k$-hop subgraph expansion based on all updated nodes $\Delta V_{(t-1,t)}$. 

The parameter $k$ is aligned with the number of layers in the local GNN. In message-passing, each layer aggregates information from immediate neighbours. Consequently, stacking $k$ layers allows node representations to incorporate information from nodes within their $k$-hop neighbourhood. 
 
The $k$-hop neighbours of all updated nodes at time $t$ form a subgraph, denoted as $g^k_{t}$.
\begin{equation*}
g^k_t = G_t[\mathcal{V}(g_t^k) ], \quad  \mathcal{V}(g_t^k) = \bigcup_{v \in \Delta V_{(t-1,t)}} \mathcal{N}_{G_t}^k(v)
\end{equation*}
where $\mathcal{N}_{G_t}^k(v)$ denotes the $k$-hop neighbourhood of node $v$ in $G_t$.

\subsubsection{Step 2: First GNN layer} \label{sesh:learning2}
The subgraph $g^k_{t}$ is passed through the first GNN layer, producing updated node embeddings, denoted by $\textbf{H}^{k, (1)}_t$, where the superscript $1$ indicates the first GNN layer.

A GNN layer updates node embedding via iterative neighbourhood aggregation. Mathematically, we write:
\begin{equation*}
H^{k,(1)}_{t,v} = \sigma\left( W^{(1)} \cdot \text{AGG}^{(1)}\left( \left\{ H^{(0)}_{t,u} \mid u \in \mathcal{N}(v) \cup {\{v\}} \right\} \right) \right)
\end{equation*}
where $H^{(0)}_{t,u} = X_{t,u}$ is initial node feature of node $u$, $\text{AGG}^{(1)}$ is an aggregation function (e.g., mean, sum) at the first layer, $W^{(1)}$ are learnable weights, and $\sigma$ is a non-linear activation function. 

The collection of node embeddings for all nodes $v$ in $g^k_t$ is denoted as:
\begin{equation*}
\mathbf{H}^{k,(1)}_t = \bigcup_{v \in g^k_t} H^{k,(1)}_{t,v}
\end{equation*}

\subsubsection{Step 3: GRU module} \label{sesh:learning3}
A GRU module captures the temporal relationships between consecutive graph updates. It takes as input the output of the first GNN layer, $\textbf{H}^{k,(1)}_{t}$, along with the hidden states from the previous update for the corresponding nodes in $g^k_{t}$, denoted as $\textbf{H}^{k,(1)}_{t-1}$. The GRU outputs updated hidden states, $\widehat{\textbf{H}}^{k,(1)}_{t}$, for the nodes in $g^k_{t}$.
\begin{equation*}
\widehat{\textbf{H}}^{k,(1)}_t = \text{GRU}\left( \textbf{H}^{k,(1)}_t, \textbf{H}^{k,(1)}_{t-1} \right)
\end{equation*}
The GRU takes the current and previous hidden states and outputs the temporally updated states.

\subsubsection{Step 4: Assembling the full hidden states} \label{sesh:learning4}
The updated hidden states, $\widehat{\textbf{H}}^{k,(1)}_{t}$, serve two purposes: (1) they are fed into the next GNN layer, and (2) they are combined with the hidden states of nodes not in $g^k_{t}$ to reconstruct the complete node representation for snapshot $t$ for subsequent graph snapshot.

Since nodes not in $g^k_{t}$ implies nodes that remains unchanged from the previous snapshot, their node states remains the same and can be directly reused from snapshot $t-1$. We denote the hidden states of these unaffected nodes by $\widehat{\textbf{H}}^{\overline{k},(1)}_{t-1} = \widehat{\textbf{H}}^{(1)}_{t-1} \setminus \widehat{\textbf{H}}^{k,(1)}_{t-1}$.
\begin{equation*}
\widehat{\textbf{H}}^{(1)}_t = \widehat{\textbf{H}}^{k,(1)}_t \cup \widehat{\textbf{H}}^{\overline{k},(1)}_{t-1}
\end{equation*}

\section{Node Embedding Exchange Scheme}

\subsection{Proof of Theorem \ref{thm:thm1}} \label{appen:proof}
We prove the above formula using the observation of the adjacency matrix of the global graph and the subgraphs to find the additional aggregation required to the submitted feature given by the cross-client neighbours. 

According to the formulation in Section \ref{sesh:formulation}, ignoring the timestamp $t$, we consider the global graph's adjacency matrix $\mathbf{Adj}$, where the $i^{th}$ row and $j^{th}$ column of this matrix is $1$ if there exists $(i,j) \in E_t$ and otherwise 0. Note that the adjacency matrix includes self-loops as well. We also define $\mathcal{\textbf{Z}}^{(0)} = \textbf{X}$ as the 0-hop node feature matrix. 

The global 1-hop feature matrix is computed by
\begin{equation}
    \mathcal{\textbf{Z}}^{(1)} = \mathbf{Adj} \cdot \mathcal{\textbf{Z}}^{(0)}
\end{equation}
We can get the 2-hop feature matrix by

\begin{equation}\label{eq:sqr_matx}
    \mathcal{\textbf{Z}}^{(2)} = \mathbf{Adj} \cdot \mathcal{\textbf{Z}}^{(1)} = \mathbf{Adj} \cdot (\mathbf{Adj} \cdot \mathcal{\textbf{Z}}^{(0)}) = \mathbf{Adj}^2 \mathcal{\textbf{Z}}^{(0)}
\end{equation}

To form the global 1-hop feature vector of each node $i$ (i.e. $\mathcal{\textbf{z}}^{(1)}_i$), we start with its local 1-hop feature submitted by the client holding it. Let such client's subgraph be denoted by $s$. Since $\mathcal{\textbf{Z}}^{(0)}$ is 0-hop node feature matrix which is fixed across clients, we can discard $\mathcal{\textbf{Z}}^{(0)}$ when comparing the global and local feature aggregations and generalize the computation of node $i$'s global 1-hop feature as:
\begin{equation*}
    \mathbf{Adj}_i \mathcal{\textbf{Z}}^{(0)} = \mathbf{Adj}^{(s)}_i \mathcal{\textbf{Z}}^{(0)} + \textbf{c}^{\top}_i \mathcal{\textbf{Z}}^{(0)}
\end{equation*}
implies 
\begin{equation} \label{eq:1hopadj}
    \mathbf{Adj}_i = \mathbf{Adj}^{(s)}_i + \textbf{c}^{\top}_i
\end{equation}

Subtracting $\mathbf{Adj}^{(s)}_i$ from $\mathbf{Adj}_i$, $\textbf{c}^{\top}_i$ is a vector such that its $j^{th}$ entry:
$$
c_{i,j} = 
\begin{cases}
0 & \text{if node } j \text{ is in global graph } G \text{ and subgraph } s  \\
1  & \text{if node } j \text{ is in global graph } G \text{ but not in subgraph } s
\end{cases}
$$

In other words, the non-zero entries in $\textbf{c}_i$ correspond to the \textbf{cross-client neighbours} of node $i$ (i.e. $\hat{\mathcal{N}}_i$). This implies the exact 1-hop global node embeddings for node $i$ is the sum of the 1-hop feature of node $i$ (derived from $\mathbf{Adj}_{i}^{(s)}$) and the 0-hop features of all nodes $j \in \hat{\mathcal{N}}_i$ (derived from $\textbf{c}_i$):

\begin{equation*}
    H^{*,(1)}_i = H^{k,(1)}_i + \sum_{j \in \hat{\mathcal{N}}_i} H^{(0)}_j
\end{equation*}
proving the first Equation of computing 1-hop node embedding in the main text.

We now proceed to the 2-hop case by considering the cross-client neighbours $\hat{\mathcal{N}}_i$. By equating $\mathbf{Adj}_i^2$ and $(\mathbf{Adj}_i^{(s)})^2$, we obtain the additional nodes and their values needed to compute the exact 2-hop global node embedding for node $i$. 
\begin{equation} \label{eq:2hopadj}
    \mathbf{Adj}_i^2 = (\mathbf{Adj}_i^{(s)})^2 + \textbf{d}^{\top}_i
\end{equation}

To compute the values of $\textbf{d}^{\top}_i$, we rearrange Equation \ref{eq:2hopadj} and substitute Equation \ref{eq:1hopadj} to the equation:
\begin{align*}
    \mathbf{Adj}_i^2 &= (\mathbf{Adj}_i^{(s)})^2 + \textbf{d}^{\top}_i \\
    \mathbf{Adj}_i \cdot \mathbf{Adj} &= \mathbf{Adj}_i^{(s)} \cdot \mathbf{Adj}^{(s)} + \textbf{d}^{\top}_i \\
    \textbf{d}^{\top}_i &= \mathbf{Adj}_i \cdot \mathbf{Adj} - \mathbf{Adj}_i^{(s)} \cdot \mathbf{Adj}^{(s)} \\
     &= (\mathbf{Adj}_i^{(s)} + \textbf{c}^{\top}_i) \cdot \mathbf{Adj} - \mathbf{Adj}_i^{(s)} \cdot \mathbf{Adj}^{(s)} \\
     &= \mathbf{Adj}_i^{(s)} \cdot \mathbf{Adj} + \textbf{c}^{\top}_i \cdot \mathbf{Adj} - \mathbf{Adj}_i^{(s)} \cdot \mathbf{Adj}^{(s)} \\
     &= \mathbf{Adj}_i^{(s)} (\mathbf{Adj} - \mathbf{Adj}^{(s)}) + \textbf{c}^{\top}_i \cdot \mathbf{Adj}
\end{align*}
which yields two terms $\mathbf{Adj}_i^{(s)} (\mathbf{Adj} - \mathbf{Adj}^{(s)})$ and $\textbf{c}^{\top}_i \cdot \mathbf{Adj}$. 

\textbf{Consider the first term $\mathbf{Adj}_i^{(s)} (\mathbf{Adj} - \mathbf{Adj}^{(s)})$.} 

$(\mathbf{Adj} - \mathbf{Adj}^{(s)})$ implies the set of nodes that are not in subgraph $s$, while $\mathbf{Adj}^{(s)}_{i,j} = 0$ for all $j$ not in subgraph $s$, we get a single row vector filled with zeros. 

\textbf{Consider the second term $\textbf{c}^{\top}_i \cdot \mathbf{Adj}$.}

From the above proof, we know that $c^{\top}_i$ records the cross-client neighbours of node $i$ and $\mathbf{Adj}_j$ is the 1-hop neighbour of node $j$, including itself, in the global graph. 
Multiplying $\textbf{c}^{\top}_i$ by $\mathbf{Adj}$ gives a row vector $\textbf{m}$ such that each entry $m_u$ in $\textbf{m}$ implies \textbf{the number of times node $u$ has been visited by any cross-client neighbours $j \in \hat{\mathcal{N}}_i$}. 

Combining the two terms, $\textbf{d}^{\top}_i$ is simply $\textbf{c}^{\top}_i \cdot \mathbf{Adj}$ where it reveals nodes that are reachable in 2 hops from node $i$ via cross-client edges only.

We expand $\textbf{d}^{\top}_i$ to obtain the corresponding node embeddings needed to complete the 2-hop neighbourhood of node $i$. Firstly, the cross-client neighbours $j \in \hat{\mathcal{N}}_i$ provide their local 1-hop node embeddings (i.e. $\sum_{j \in \hat{\mathcal{N}}_i} H^{k,(1)}_{j}$) which are unreachable from node $i$ locally. Next, each 1-hop cross-client neighbour should visit node $i$ once in a global graph, vice versa. To record this interaction, we include $\sum_{j \in \hat{\mathcal{N}}_i} (H^{(0)}_{j} + H^{(0)}_{i})$. Finally, if there exist 2-hop cross-client neighbours, meaning the 1-hop cross-client neighbours of $j \in \hat{\mathcal{N}}_i$, it is also indirectly visited by all cross-client neighbours $j$. Therefore, we include $\sum_{q \in \hat{\mathcal{N}}_j} H_q^{(0)}$.

As a result, the node embeddings needed to complete the 2-hop neighbourhood of node $i$ in addition to $H_i^{(2)}$ are:
\begin{equation*} 
    \sum_{j \in \hat{\mathcal{N}}_i} (H^{k,(1)}_{j} + H^{(0)}_{j} + H^{(0)}_{i}) + \sum_{q \in \hat{\mathcal{N}}_j} H_q^{(0)}
\end{equation*}
proving the second equation in the main text.

\subsection{Experiments on various Embedding Synchronisation Schedules} \label{appen:ne_experi}
\begin{table*}
\centering
  \caption{F1 Score and Total Communication Time on Node Classification Tasks with varying Node Embedding Exchange Strategies}
  \label{tab:ne_strat}
  \begin{tabular}{cccc|ccc}
    \toprule
     Node Embedding& \multicolumn{3}{c|}{\textbf{F1}} & \multicolumn{3}{c}{\textbf{Total Communication Time (s)}} \\
     Exchange Strategies & DBLP5 & DBLP3 & Reddit & DBLP5 & DBLP3 & Reddit\\
    \midrule
    After first epoch, all rounds& 0.7215& 0.6558& 0.2654& 4.4994& 1.4910& 221.0935\\
    After last epoch, all rounds& \textbf{0.7273}& 0.6595& 0.2703& 4.5242& 1.5013& 221.5287\\
    After last epoch, first round& 0.7269& \textbf{0.6616}& \textbf{0.2784}& \textbf{0.4309}& 0.1446& 22.5554\\
    After last epoch, mid round& 0.7167& 0.6600& 0.2753& 0.4410& \textbf{0.1429}& \textbf{22.1005}\\
    \bottomrule
  \end{tabular}
\end{table*}
Frequent exchange of embeddings can incur significant communication overhead. To investigate the trade-off between performance and communication efficiency, we evaluate four embedding exchange strategies: 1) Exchanging after the first local epoch in every communication round; 2) Exchanging after the last local epoch in every communication round; 3) Exchanging after the last local epoch only in the first communication round; 4) Exchanging after the last local epoch at the midpoint of all the communication rounds. We conduct experiments on the above four options to measure their model performance (F1) and total communication time during node embedding exchange in Table \ref{tab:ne_strat}. 

Experimental results demonstrate that the most effective strategy is Option 3, where embeddings are exchanged only once, after the last local epoch for the first communication round, achieving a favourable balance between performance and communication cost. Comparing with Option 1 and 2, Option 3 achieves F1 scores either comparable to (DBLP5: $\Delta =$ -0.0004) or superior than (DBLP3: +0.21\%, Reddit: +3.0\% to next best) the first two options, while reducing communication time by \textbf{10.4$\times$} (DBLP5) to \textbf{9.8$\times$} (Reddit). This demonstrates that exchanging solely in the first round is sufficient to maintain performance while optimizing efficiency.

In terms of exchanging in a single round, Option 3 consistently outperforms Option 4 across all datasets, with improvements ranging from 0.2\% (DBLP3) to 1.4\% (DBLP5). Because of the low baseline performance in Reddit, even small gains in F1 are significant (1.1\% gain in F1). Since both strategies exchange for a single round, they also share comparable communication time ($\Delta \leq $ 2.0\%). This suggests that early exchanges in the first round better preserve graph structure for model learning. Thus Option 3 is strictly preferable for balancing performance and overhead.

In conclusion, Option 3 achieves the \textit{minimum necessary alignment} to stabilize learning without repeated overhead.

\subsection{Communication Cost of Node Embedding Exchange}\label{appen:ne_comcost}
We formalise the communication cost $C$ for node embedding exchanges as:
\begin{align*}
    C &= N \cdot (Cut\_Edge\_Ratio \cdot |\mathcal{E}| \cdot D_{bytes}) + T \cdot L
\end{align*}
where
\begin{itemize}
    \item $N$ is the number of exchange rounds
    \item $Cut\_Edge\_Ratio$ is the average cut edge ratio across snapshots
    \item $|\mathcal{E}|$ is the total number of edges in the dataset
    \item $D_{bytes}$ is the node embedding dimension in bytes (dimension $D \times 4$ for float32)
    \item $T$ is the number of snapshots
    \item $L$ is the fixed latency per round
\end{itemize}
While we picked Option 3 as the node embedding exchange strategy and \texttt{MinCut\_First} as the graph partitioning configuration, we yield the following statistic in Table \ref{tab:comcost_param}. (Here we set $L \approx 0.1s$)

\begin{table}
    \centering
    \begin{tabular}{c|c|c|c}
    \toprule
         \textbf{Parameters} & \textbf{DBLP5}& \textbf{DBLP3}& \textbf{Reddit}\\
         \midrule
         $N$&  1&  1& 1\\
         $Cut\_Edge\_Ratio$ & 0.0006&0.0006& 0.036\\
         $|\mathcal{E}|$ &  42K&  24K& 244K\\
         $T$& 8& 8&8\\
         $D$& 100 & 100 & 20\\
         $D_{bytes}$& 400 & 400 & 80\\
         \bottomrule
    \end{tabular}
    \caption{Parameter values for each dataset}
    \label{tab:comcost_param}
\end{table}
\begin{table}
    \centering
    \begin{tabular}{c|c|c|c}
    \toprule
     &  \textbf{DBLP5}&  \textbf{DBLP3}& \textbf{Reddit}\\
     \midrule
     $C$& 10.08KB + 0.8s & 5.76KB + 0.8s & 702.72KB + 0.8s \\
     \bottomrule
    \end{tabular}
    \caption{Communication Cost $C$ of Node Embedding Exchange for each dataset}
    \label{tab:com_result}
\end{table}
Therefore, the calculation of communication cost $C$ for each dataset is shown in Table \ref{tab:com_result}. For DBLP5 and DBLP3, costs remain minimal, aligning with their minimal cut edges and effective \texttt{MinCut\_First} partitioning ($\leq$ 0.01 cut edge ratios). In contract, Reddit incurs 702.72KB of traffic, 70x higher than DBLP5, despite its lower cut edge ratio. This arises from Reddit's massive edge count ($\approx 244K$ edges), where even a small ratio yields 8788 absolute cuts. The observed 112x longer exchange time (56s vs. 0.5s) further reflects this network protocol overhead.

\section{Experiment Details} \label{appen:exp_detail}

\subsection{Details of Learning Tasks}
(1) Link prediction aims to predict future edges based on existing node and edge information. After learning node embeddings from both current and historical graph structures, the likelihood of an edge \( (u, v) \) is obtained by computing the dot product of the embeddings of nodes \( u \) and \( v \). Following standard practice \cite{ROLAND}, we evaluate the prediction model using Mean Reciprocal Rank (MRR) and Mean Average Precision (MAP), supplemented by test average for robustness.
MRR quantifies how highly the \textit{first} correct link is ranked for each query node, crucial for applications such as recommender system where users rely on the top-ranked results. MAP, in contrast, assesses the ranking quality of all correct links, reflecting overall prediction consistency.

(2) Node classification involves assigning labels to nodes based on their features and graph structure. We follow the experimental setup of FedDGL \cite{FedDGL}, which also addresses federated dynamic graph learning, and use the F1 score as the evaluation metric. Compared to accuracy, F1 gives a more nuanced view by penalising false positives and false negatives. Instead of prioritising on majority labels, F1 shows whether the model generalizes well to minority and edge-case nodes, which is crucial in real-life graphs.

Additional details: For all link prediction tasks, we apply negative sampling, generating an equal number of negative edges as positive ones for each client to ensure a balanced dataset.

\subsection{Temporal Evaluation Protocol} \label{appen:temporal_eval}
Our training pipeline follows a \textbf{temporally shifted evaluation scheme}, identical to the live-update evaluation in FedDGL \cite{FedDGL}. Specifically, the global model is trained using snapshot $t$, validated on snapshot $t+1$, and tested on snapshot $t+2$. During training on each snapshot, clients do not explicitly observe cross-client edges. Such edges are only implicitly incorporated through the node embedding exchange scheme. This protocol resembles closely to real-world dynamic graph systems, where snapshots arrive frequently and models are typically trained on historical data to make predictions on future graph states.

\subsection{Dataset Descriptions}
We conduct experiments using six standard datasets and two large-scaled temporal benchmark datasets\cite{huang2024tgb2}.

(1) UCI-Message (UCI) is a link prediction dataset consisting of private messages exchanged among students on an online social network system, used to study social dynamics and message flow \cite{UCI}. Each edge $(u,v,t)$ in the dataset implies user $u$ sent a private message to user $v$ at time $t$. With a total time span of 193 days, we made use of a 5-day window which results in 39 time steps.

(2) Bitcoin-OTC (OTC) is a link prediction dataset representing who-trusts-whom networks of users trading on the OTC platform, where edges indicate trust relationships between users \cite{OTC}. There are 138 such networks of sequential time steps.

(3) DBLP3 and (4) DBLP5 are node classification datasets derived from the DBLP citation network, focusing on a subset of computer science research papers and their citation relationships \cite{Social_Network}. Both datasets consist of 10 graph snapshots, each with an adjacency matrix, feature matrix and one-hot encoded label matrix.

(5) Reddit is a node classification dataset constructed from a large social network graph on Reddit, where nodes represent users and edges indicate co-participation in discussions or shared subreddit memberships \cite{Reddit}. Similarly, Reddit has 10 graph snapshots with the same attributes as DBLP5 and DBLP3.

(6) AS-733 is a large-scaled link prediction dataset comprising 733 daily subgraphs of routers which span an interval of 785 days from November 8 1997 to January 2 2000, formed as Autonomous Systems (AS). Each AS exchanges traffic flows with some neighbours. In contrast to the existing networks mentioned, where nodes and edges only get added and not deleted over time, this dataset also exhibits deletion of nodes and edges over time \cite{as-733}. 

(7) tgbl-coin is a cryptocurrency transaction dataset based on the Stablecoin ERC20 transactions dataset. Each node is an address and each edge represents the transfer of funds from one address to another at a time. The network starts from April 1st, 2022, and ends on November 1st, 2022, and contains transaction data of 5 stablecoins and 1 wrapped token. This duration includes the Terra Luna crash where the token lost its fixed price of 1 USD.

(8) tgbn-reddit is a users and subreddits interaction network. Both users and subreddits are nodes and each edge indicates that a user posted on a subreddit at a given time. The dataset spans from 2005 to 2019.

Table \ref{tab:data_stat} shows the number of nodes, edges, classes and the type of learning task for each dataset.
\begin{table}
    \centering
    \small
    \begin{tabular}{c|c|c|c|c}
    \toprule
         \textbf{Dataset}&  \textbf{\# Nodes}&  \textbf{\# Edges}& \textbf{\# Classes} & \textbf{Learning Task}\\
         \midrule
         DBLP5&  6,606&  42,812& 5 & Node Classification\\
         DBLP3&  4,257&  23,538& 3 & Node Classification\\
         Reddit&  8,291&  244,552& 4 & Node Classification\\
         OTC&  6,005&  355,096& 2 & Link Prediction\\
         UCI&  1,899&  575,971& 2 & Link Prediction\\
         AS-733&  7,716&  11,965,533& 2 & Link Prediction\\
         \midrule
         tgbl-coin&  638,486&  22,809,486& 2 & Link Prediction\\
         tgbn-reddit&  11,766&  27,174,118& 2 & Node Classification\\
         \bottomrule
    \end{tabular}
    \caption{Dataset Statistics}
    \label{tab:data_stat}
\end{table}

\subsection{Description of Existing Graph Partitioning Algorithms} \label{appen:exist_gpa}
Label-skew and Louvain are used in FedDGL \cite{FedDGL}, where Label-skew simulates non-IID data scenarios and Louvain is a community detection algorithm to group densely connected nodes. Metis is used in Aligraph \cite{Aligraph} to partition vertices into subgraphs with balanced node counts while minimizing edge cuts. As mentioned, these existing approaches fail to handle disconnected graphs. For instance, Metis assigns all isolated nodes to one subgraph while the remaining connected components to another to minimize cut edges. This leads to an imbalanced workload across client, thereby reducing computational efficiency in distributed learning. To ensure fair comparison, we also introduce synthetic edges to connected components before employing the existing partitioning approaches.

\subsection{Baseline Model and Method Descriptions}
Existing dynamic graph learning models, GCRN-Baseline, GCRN-GRU, GCRN-LSTM, EvolveGCN and ROLAND, have been primarily evaluated on link prediction tasks, while collaborative graph learning models, D-FedGNN, FedAvgDyn, FedProtoDyn and FedDGL, have focused on node classification. Since DG-CoLearn supports both link prediction and node classification, we evaluate it against models from both domains accordingly.

Among the collaborative learning baselines, FedAvgDyn and FedProtoDyn are adapted from EvolveGCN \cite{EvolveGCN}, where the local training model at each client is EvolveGCN, while the aggregation method at the server follows FedAvg and FedProto, respectively.

Each result is obtained by taking an average of 3 trials of experiments.

\subsection{Additional Findings from ROLAND's performance}

Table \ref{tab:LP} reveals that while ROLAND \cite{ROLAND} achieves high MAP (indicating broad correctness), its lower MRR suggests a ranking bias: high-degree nodes dominate top predictions, even when less relevant. This aligns with known challenges in graph-based ranking, where structural properties can overshadow semantic relevance.

\subsection{Experiment Result on tgbn-reddit} \label{appen:tgb}
The full result of large scaled temporal graph tgbn-reddit for node classification task is illustrated in Table \ref{tab:tgbn_reddit}, and tgbl-coin for link prediction task is reported in Table \ref{tab:tgbl_coin}.

\begin{table}
\centering
\small
\caption{Efficiency and performance on tgbn-reddit (node classification). DG-CoLearn improves both system efficiency and predictive quality.}
\label{tab:tgbn_reddit}
\begin{tabular}{lccccc}
\toprule
\textbf{Method} 
& Training Time (s)& Partition Time (s)& Cut Edge Ratio (\%)& NE Exchange Time (s)& F1 Score\\
\midrule
FedDGL 
& 117.28 $\pm$ 23.31
& 11.44 $\pm$ 4.85
& 11.7 $\pm$ 3.69 
& N/A 
& 0.0929 $\pm$ 0.0513 \\

DG-CoLearn 
& \textbf{53.56 $\pm$ 10.58}
& \textbf{3.88 $\pm$ 1.52}
& \textbf{2.72 $\pm$ 1.27}
& 3.72 $\pm$ 1.48 
& \textbf{0.1040 $\pm$ 0.0468} \\

\midrule

Improvement
& 2.19$\times$
& 2.95$\times$
& 76.8\% $\downarrow$
& --
& +11.98\% \\

\bottomrule
\end{tabular}
\end{table}

\begin{table}
\centering
\small
\caption{Scalability results on tgbl-coin (link prediction). DG-CoLearn achieves substantial speedups and reduced communication overhead.}
\label{tab:tgbl_coin}
\begin{tabular}{lccccc}
\toprule
\textbf{Method} & $\downarrow$ Training Time (s) & $\downarrow$ Partition Time (s) & $\downarrow$ Cut Edge Ratio (\%) & $\downarrow$ NE Exchange Time (s) & Completed \\
\midrule
Full Retrain
& 5126.9 $\pm$ 483.9
& 1553.4 $\pm$ 118.5
& 3.35 $\pm$ 0.89
& 279.8 $\pm$ 23.05 
& No \\

DG-CoLearn 
& \textbf{151.6 $\pm$ 11.7}
& \textbf{30.09 $\pm$ 3.40} 
& \textbf{2.30 $\pm$ 0.81}
& \textbf{10.22 $\pm$ 3.11}
& Yes \\
\midrule

Improvement
& 33.8$\times$
& 51.6$\times$
& 31.3\% $\downarrow$
& 27.4$\times$
& -- \\

\bottomrule
\end{tabular}
\end{table}

\subsection{Compute Resources}
All experiments were conducted on a single workstation equipped with an NVIDIA RTX 4090 GPU (24GB VRAM). Per-dataset training, partitioning, and node embedding exchange timings reported in Tables 4 and 10 and Figures 4–5 were measured on this hardware. The total compute across reported experimental runs was on the order of approximately 10 GPU-hours. Preliminary experiments and hyperparameter exploration during development required additional compute beyond the reported runs.

\section{Details on Ablation Studies} \label{appen:abla}

\subsection{Node Embedding Exchange Ablation}
\subsubsection{Baseline Selection Rationale}
For the ablation study, we focus on isolating the effect of \textbf{node embedding exchange} under the same privacy and partitioning constraints as DG-CoLearn. Specifically, we select the following baselines:
\begin{itemize}
    \item \textbf{No Node Embedding Exchange}: This baseline disables any form of cross-client information sharing. Each client performs local GNN training using only intra-client edges, and the server aggregates model parameters via FedAvg without exchanging node-level representations. This setting serves as a lower bound, highlighting the impact of ignoring cross-client dependencies entirely.
    \item \textbf{Incremental FedGCN}: FedGCN is a representative federated learning method that exchanges aggregated node features to approximate cross-client neighbourhood information. We select FedGCN as it is one of the few methods that explicitly addresses missing neighbours under client-oblivious privacy constraints, making it the closest comparable baseline for our embedding exchange mechanism.
\end{itemize}

\subsubsection{Implementation of the Incremental FedGCN Baseline}
FedGCN was originally proposed for static graphs and is not directly applicable to dynamic graph learning or DGL-style temporal processing. To enable a fair comparison, we adapt FedGCN to our experimental setting as follows.

At each snapshot $t$, each client $m$ computes local feature aggregations for its nodes based on its local adjacency matrix which now only includes nodes within 2-hops of all changed nodes:
\begin{equation*}
    \mathcal{F}^{(m)}_i = \sum_{j \in \mathcal{N}^{(m)}_i}\mathbf{Adj}^{(m)}_{i,j} \cdot x_j
\end{equation*}
where $\mathcal{N}^{(m)}_i$ denotes the intra-client neighbours of node $i$, and $x_j$ is the input features of node $j$.

Upon receiving all feature aggregations from the clients, the server, who has access to the full graph topology, aggregates cross-clients contributions by summing feature aggregations from all relevant clients.
\begin{equation*}
    \mathcal{F}_i^* = \sum_{m \in M} \mathcal{F}^{(m)}_i
\end{equation*}
where $M$ is the set of clients with nodes that are cross-client neighbours $\hat{\mathcal{N}}_i$ of node $i$, and within the update subgraph.

The resulting $\mathcal{F}^*_i$ is then used as the updated input feature for node $i$ in the next training round. This procedure mirrors the original FedGCN aggregation logic while allowing incremental execution across snapshots.

Importantly, this approach replaces original node features with aggregated features and does not preserve intermediate representations across GNN layers. As a result, higher-order message passing interactions between cross-client neighbours are not explicitly modelled.

\subsubsection{Setup of Node Embedding Exchange Ablation Study}
To explicitly evaluate the effectiveness of our node embedding exchange mechanism, we design a server-side evaluation task restricted to cross-client edges. This setting directly assesses how well different methods recover relational information that is not locally observable by any single client.

We consider a \textbf{cross-client edge prediction (CCE) task}, where the objective is to predict the existence of edges connecting nodes that belong to different clients. \textit{Positive samples} consist of ground-truth cross-client edges from the current snapshot. In alignment with the temporal evaluation protocol in Appendix \ref{appen:temporal_eval}, this choice reflects a realistic inference scenario in which the server evaluates how well the model trained on snapshot $t$ has internalised cross-client structural relationships present at the timestamp, despite such edges never being directly visible to any single client during training. These edges are extracted from the global edge set by filtering for node pairs whose endpoints belong to different clients under the fixed node-to-client assignment. Only edges satisfying this criterion are included in the CCE evaluation. \textit{Negative edges} are sampled by first selecting two distinct clients uniformly at random, then sampling one node from each client, and finally rejecting pairs that correspond to existing edges. For each snapshot, we sample an equal number of negative cross-client edges as positive ones. This procedure ensures that all evaluated node pairs are cross-client while avoiding false negatives.

We report the Mean Average Precision (MAP) on the constructed CCE evaluation test. This is because the primary objective of CCE prediction task is to assess the overall quality of the learned ranking over candidate cross-client edges.

\subsection{Incremental Graph Processing Ablation}
\subsubsection{Baseline Selection Rationale}
To assess the efficiency gains introduced by \textbf{incremental processing}, we introduce an additional ablation variant termed \textbf{Full-graph Node Embedding Exchange}.

This variant preserves the same node embedding reconstruction mechanism as DG-CoLearn, but disables incremental snapshot processing. In particular, node embedding exchange is performed for all cross-client nodes in the graph at every snapshot, rather than being restricted to updated cross-client nodes. Overall speaking, this baseline is served as a upper bound on communication and computation cost.

\section{Additional Experiments -- Evaluating Catastrophic Forgetting} \label{appen:catforget}
\begin{figure}[t]
  \centering
  \includegraphics[width=0.5\linewidth]{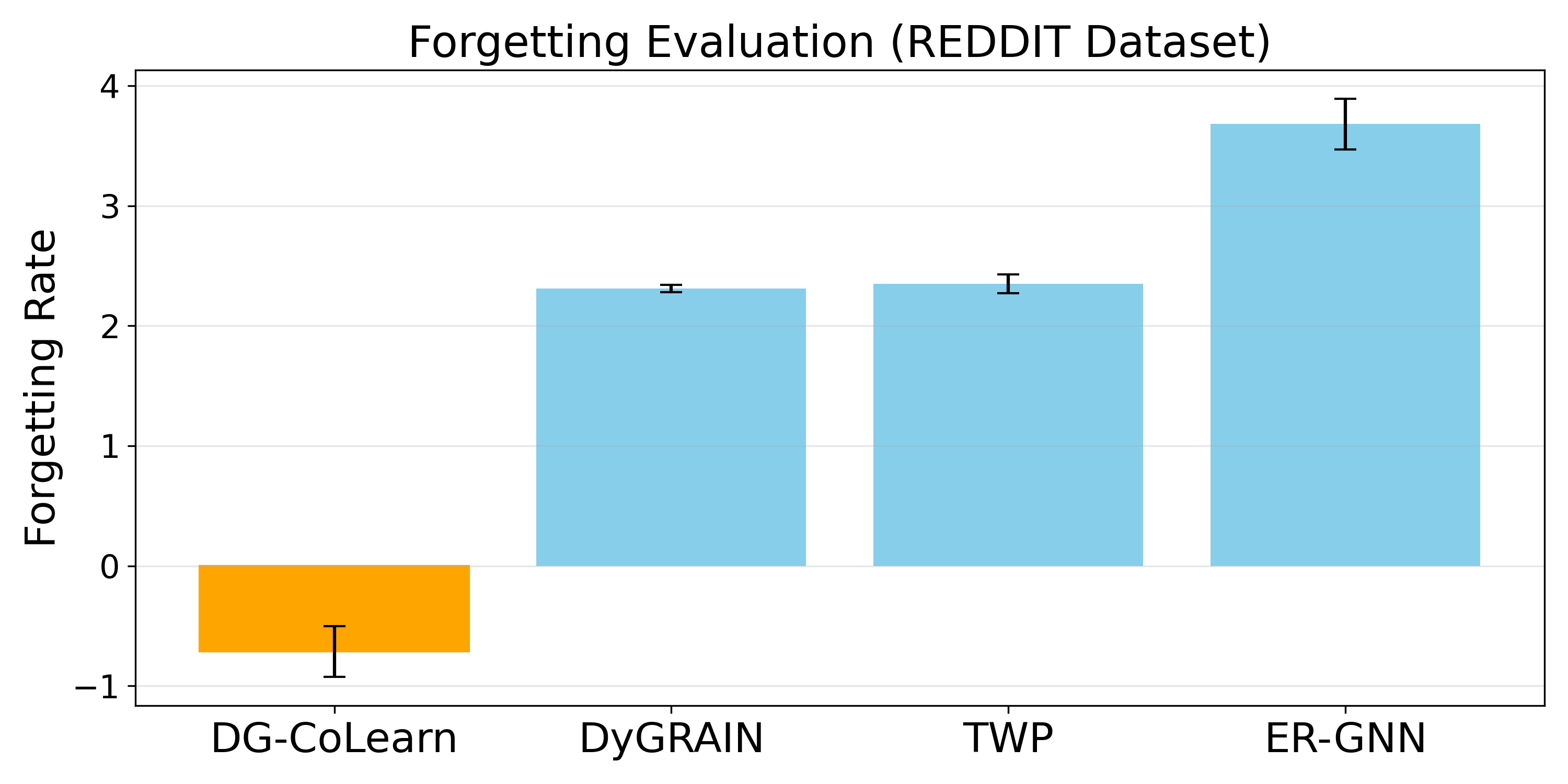}
  \caption{Measuring Catastrophic Forgetting across incremental learning methods}
  \label{fig:forgetting}
\end{figure}

A key challenge in incremental learning is \textit{catastrophic forgetting}, where the model loses previously acquired knowledge while adapting to new data. To evaluate the robustness against this phenomenon, we adopted the same practice as DyGRAIN and use \textit{Forgetting} \cite{forgetting} to measure the degree of catastrophic forgetting. 

\subsection{Baseline Models}
Incremental graph learning methods broadly fall into three categories: replay-based methods, regularization-based methods, and influence-based methods.

Replay-based methods retrain historical data stored in a memory. These sampled past data are mixed into the current data during training to relieve catastrophic forgetting. Representative method is ER-GNN \cite{ER-GNN}. Regularisation methods penalise the change of \textit{important} parameters, while the importance is typically measured by gradient distributions of the data. TWP is a recent work that quantifies importance via previous tasks and topological attention between adjacent nodes \cite{TWP}. Influence-based methods focus on updating parts of the graph that are actually affected by new information. DyGRAIN \cite{DyGRAIN} uses structure-aware computation to select nodes that are influenced by new edges and vulnerable nodes that might forget past knowledge. While their approach is similar to DG-CoLearn, they only consider addition of nodes and edges, ignoring the potential of element deletion. 

\subsection{Experimental Results}
Figure \ref{fig:forgetting} illustrates the Forgetting evaluation across various existing incremental learning approaches, DyGRAIN, TWP and ER-GNN. While our evaluation focuses on Reddit, this dataset is the largest node classification benchmark, which leads to the highest risk of overwriting previous knowledge as a vast amount of new graph data flows in. Therefore, it effectively tests scalability and memory efficiency of incremental methods.

As the existing methods have positive Forgetting values, DG-CoLearn results in a \textbf{negative} value, which implies the model not only remembered the previous knowledge, but also \textbf{improved} its performance after learning new tasks. This is commonly regarded as \textit{forward transfer}, which is a strong indicator of \textbf{effective knowledge reuse}. DG-CoLearn's ability to achieve negative forgetting stems from its targeted learning of changed graph components and modular integration with prior knowledge. By isolating and training only the dynamic components, we prevent interference with unchanged components. While the previous frozen model retains 100\% of the past task (zero forgetting), the combination of the two models additively enhance the model's performance on both new and old data, effectively turning forgetting into improvement.

\section{Privacy Guarantee} \label{appen:privacy}

\subsection{Privacy Model Formulation}
DG-CoLearn is designed under the standard \textbf{trusted-server, client-oblivious} privacy model, widely used in intra-graph distributed learning (e.g., FedSage+, LLCG). Specifically: 

(1) The server is the source of graph data collection and hence has access to the global topology. This is practically motivated: the server harnesses the collective computing power of the clients to achieve distributed learning, which speeds up the training and potentially enables the training of larger-scale graphs compared to training all graphs on a single server. Furthermore, the server is a trusted coordinator to the clients. Under this collaborative graph learning setting, privacy concern mainly focuses on whether clients are aware of the graph data that are not allocated to them by the server. This is a standard in related literature such as FedSage+, and recent work FedGNNLDP \cite{FedGNNLDP}.

(2) Once the server partitions the graph and distributes subgraphs to clients, each client only has access to their local subgraph and has no visibility to other client's. In particular, no client can infer cross-client edges, other client's raw node features, or their local topologies; Clients cannot directly communicate with other clients, and only receive their own updated embeddings which are computed by the server via a well-formulated aggregation (Equations \ref{eq:glob1}–\ref{eq:glob2}, Theorem \ref{thm:thm1}).

Formally, we have:
\begin{proposition}[Client Privacy]
    For client $m$ at snapshot $t$, the view of the client under DG-CoLearn consists only of the embeddings $H^{*,(l)}_{t,i}$ for $i \in V_t^{(m)}$ (all vertices held by client $m$ at time $t$). No other client’s raw features or edges are revealed.
\end{proposition}

\subsection{Privacy Attack Analysis}

There are three main families of recently studied attacks that can potentially leak graph information from model gradients or node embeddings. We analyse each attack and explain how these attacks cannot be directly applied to DG-CoLearn.

\begin{itemize}
    \item \textbf{Membership Inference Attack (MIAs):} \cite{MIA1} showed that GNN models leak information about member nodes they were trained on, and critically, that structural information is the major contributing factor to this leakage. The adversary queries the trained model and observes posterior probabilities. The most recent work, GLO-MIA \cite{MIA2}, pushes this further to a label-only setting where the attacker only sees prediction labels, not probability vectors. It perturbs the target graph's features and measures prediction stability. The Structure MIA (SMIA) \cite{MIA3} goes beyond node membership to infer whether a set of nodes forms a particular subgraph structure (k-clique or k-hop graph) in the training graph. The adversary uses posterior classification probabilities of specific nodes through a shadow-model approach. \cite{MIA4} further showed that attention mechanisms combined with topological features improve node-level MIA, exploiting the graph structure that GNNs inherently encode into their predictions.

    However, in DG-CoLearn's client-oblivious design, a client never sees another client's graph data, making such querying unavailable.

    \item \textbf{Gradient Inversion Attacks (GIAs):} SoK on gradient leakage \cite{GIA1} systematizes this threat: most GIAs presume an honest-but-curious server that passively analyses shared model gradients, while stronger attacks assume a malicious server that can craft or modify model parameters to extract more information. The DLG technique \cite{GIA2} reconstructs original training data by iteratively optimizing randomly initialised inputs to match shared gradients. Even when real gradients aren't directly accessible, methods like GLAUS can approximate the gradient from leaked indices and signs, reconstructing sensitive data \cite{GIA3}.
    
    In GIAs, the attacker intercepts raw gradients or model parameter updates that clients share during distributed training rounds. However, in DG-CoLearn, clients only share updates directly with the server, which is designed to be trusted. Therefore, attacks originating from the server fall outside our privacy model.

    \item \textbf{Graph Unlearning Privacy Leakage:} \cite{GUN1} pointed out that the difference between output distributions of models before and after unlearning can lead to additional information leakage of forgotten data, the prediction changes effectively expose membership information about unlearned samples. Recent work \cite{GUN2} also shows that even if an MIA suggests a datapoint was successfully unlearned, residual traces may remain, allowing adversaries to recover sensitive information through data poisoning vectors.

    In summary, these rely on comparing model states before and after data removal. However, clients do not have multiple model snapshots corresponding to pre- and post-unlearning states; the trusted server controls the model lifecycle and does not expose such state pairs to clients. Therefore, these comparison-based leakage channel required by existing unlearning attacks is therefore unavailable.
\end{itemize}

\subsection{Mathematical Analysis: Inability of a Colluding Client to Reconstruct Private Subgraphs by Shared Embeddings}

While the above attacks do not directly transfer, there might still exist a potential that colluding clients could infer coarse structural properties of other clients via server-shared node embeddings. Below we present a rigorous analysis to show that such information is insufficient to reconstruct private subgraphs.

\subsubsection{Threat Model}

We consider the threat model where client $m$ is adversarial and attempts to reconstruct the private subgraph $G_t^{(m')} = (V_t^{(m')}, E_t^{(m')}, X_t^{(m')})$ of another client $m' \neq m$. The server is trusted. Client $m$ has access to:
\begin{itemize}
    \item Its own subgraph $G_t^{(m)}$, including $\mathbf{Adj_t}^{(m)}, X_t^{(m)}$ and $Y_t^{(m)}$.
    \item The additional node embeddings needed to obtain the exact 1-hop and 2-hop embeddings (given by Equations \ref{eq:glob1}–\ref{eq:glob2}) received from the server for each of its border nodes $i$. We formally define these additional embeddings as:
    \begin{itemize}
        \item $\delta_i^{(1)} = \sum_{j \in \widehat{\mathcal{N}}_i} H^{(0)}_{t,j}$
        \item $\delta_t^{(2)} = \sum_{j \in \widehat{\mathcal{N}}_i} (H^{k,(1)}_{t,j} + H^{(0)}_{t,j} + H^{(0)}_{t,i}) + \sum_{q \in \widehat{\mathcal{N}}_j} H^{(0)}_{t,q}$

    respectively.
    \end{itemize}
    \item The aggregated global model parameters, which are shared across all clients.
\end{itemize}

\subsubsection{Formal Analysis}

\begin{proposition}[Unidentifiability of Cross-Client Neighbour Identity] \label{prop:1}
    Given only the 1-hop correction $\delta_i^{(1)} \in \mathbb{R}^d$, client $m$ cannot determine the identify or cardinality of the cross-client neighbour set $\widehat{\mathcal{N}}_i$.
\end{proposition}
\begin{proof}
    The correction $\delta_i^{(1)} = \sum_{j \in \widehat{\mathcal{N}}_i} H^{(0)}_{t,j}$ is a sum of $|\widehat{\mathcal{N}}_i|$ feature vectors in $\mathbb{R}^d$. Client $m$ observes the aggregate $\delta_i^{(1)}$ but knows neither $|\widehat{\mathcal{N}}_i|$ nor the individual summands.

    For the summation function $\delta_i^{(1)} = \sum_{j \in \widehat{\mathcal{N}}_i} H^{(0)}_{t,j}$, the preimage set 
    $$\mathcal{S}(\delta_i^{(1)}) = \left\{ (S, \{H_{t,j}^{(0)}\}_{j \in S}) : S \subseteq V_t \setminus V_t^{(m)}, \sum_{j \in \widehat{\mathcal{N}}_i} H^{(0)}_{t,j} = \delta_i^{(1)} \right\}$$
    is infinite (for continuous feature spaces) or combinatorially large (for discrete approximations). Without additional side information, the decomposition is not unique, and hence the individual node identities and features are unrecoverable.
\end{proof}

\begin{proposition}[Unrecoverability of Remote Graph Topology]\label{prop:2}
    Client $m$ cannot infer any edge $(j, j') \in E_t^{(m')}$ that is internal to another client $m'$. 
\end{proposition}
\begin{proof}
    We examine every component of the information available to client $m$ and show none reveals intra-client edges of $m'$.

    \textbf{The 1-hop correction} depends only on the raw node features $H_{t,j}^{(0)} = X_{t,j}$ of cross-client neighbours. Raw features $X_{t,j}$ are input-level attributes (independent of graph topology). Therefore, $\delta_i^{(1)}$ contains no information about edges within $G_t^{(m')}$.

    \textbf{The 2-hop correction} contains the term $\sum_{j \in \widehat{\mathcal{N}}_i} H^{k,(1)}_{t,j}$, where $H^{k,(1)}_{t,j}$ is the 1-hop embedding of node $j$ computed locally computed by client $m'$:
    \begin{equation*}
    H^{k,(1)}_{t,j} = \sigma\left( W^{(1)} \cdot \text{AGG}^{(1)}\left( \left\{ H^{(0)}_{t,u} \mid u \in \mathcal{N}^{(m')}(j) \cup {\{j\}} \right\} \right) \right)
    \end{equation*}
    from section \ref{sesh:learning2}.

    This embedding encodes the \textbf{aggregated} neighbourhood of $j$ within $G_t^{(m')}$. However, client $m$ receives only the \textbf{sum} $\sum_{j \in \widehat{\mathcal{N}_i}} H_{t,j}^{k,(1)}$ , not individual $H_{t,j}^{k,(1)}$ vectors. We now show this sum is insufficient to recover the local adjacency of any node $j$.

    \textbf{\textit{Sub-claim 2a (Aggregation irreversibility).}} Even if client $m$ could isolate a single $H_{t,j}^{k,(1)}$, it cannot recover $\mathcal{N}^{(m')}(j)$. When $\text{AGG}^{(1)}$ is mean aggregation:
    \begin{equation}
        \text{AGG}^{(1)} = \frac{1}{|\mathcal{N}^{(m')}(j)|+1} \sum_{u \in \mathcal{N}^{(m')}(j) \cup \{j\}} H_{t,u}^{(0)}
    \end{equation}
    Recovering $\{H_{t,u}^{(0)}\}_{u \in \mathcal{N}^{(m')}(j)}$ from $H_{t,j}^{k,(1)}$ requires inverting $\sigma \circ W^{(1)}$ and then decomposing the resulting sum. This is:
    \begin{itemize}
        \item \textbf{Non-invertible} when $\sigma$ is ReLU, since information is destroyed for negative pre-activations.
        \item \textbf{Under-determined} even if $\sigma$ and $W^{(1)}$ were invertible: the mean/sum of $|\mathcal{N}^{(m')}(j)|$ unknown vectors in $\mathbb{R}^d$ admits infinitely many decompositions when $|\mathcal{N}^{(m')}(j)| > 1$.
        \item \textbf{Degree-ambiguous}: the degree $|\mathcal{N}^{(m')}(j)|$ is unknown to client $m$, compounding the under-determination.
    \end{itemize}
    
    \textbf{\textit{Sub-claim 2b (Summation over border nodes).}} Client $m$ receives $\sum_{j \in \hat{\mathcal{N}}_i} H_{t,j}^{k,(1)}$, not individual embeddings. By the argument in Proposition~\ref{prop:1}, decomposing this sum into per-node contributions is itself infeasible without knowing $|\hat{\mathcal{N}}_i|$ or individual feature vectors.
    
    Therefore, the 2-hop correction reveals neither the identity of $j$'s local neighbours nor the edges between them.
\end{proof}

\begin{proposition}[Infeasibility of Client--Client Collusion]
\label{prop:3}
Even if two clients $m$ and $m''$ collude, they cannot reconstruct the subgraph $G_t^{(m')}$ of a third client $m'$.
\end{proposition}

\begin{proof}
Colluding clients $m$ and $m''$ can pool their respective corrections:
\begin{equation}
    \{\delta_i^{(1)}, \delta_i^{(2)}\}_{i \in B^{(m)}} \;\cup\; \{\delta_i^{(1)}, \delta_i^{(2)}\}_{i \in B^{(m'')}}
\end{equation}
where $B^{(m)}$ and $B^{(m'')}$ are the respective border node sets.

\paragraph{Case 1: Non-overlapping cross-client neighbourhoods.} If $\hat{\mathcal{N}}_i \cap \hat{\mathcal{N}}_{i'} = \emptyset$ for all $i \in B^{(m)},\, i' \in B^{(m'')}$, then the corrections provide independent aggregated views of disjoint subsets of $V_t^{(m')}$, and the arguments of Propositions~\ref{prop:1}--\ref{prop:2} apply independently to each.

\paragraph{Case 2: Overlapping cross-client neighbourhoods.} Suppose there exists $j \in V_t^{(m')}$ such that $j \in \hat{\mathcal{N}}_i$ for $i \in B^{(m)}$ and $j \in \hat{\mathcal{N}}_{i'}$ for $i' \in B^{(m'')}$. Then the colluders observe:
\begin{align}
    \delta_i^{(1)} &= H_{t,j}^{(0)} + \sum_{j' \in \hat{\mathcal{N}}_i \setminus \{j\}} H_{t,j'}^{(0)} \\
    \delta_{i'}^{(1)} &= H_{t,j}^{(0)} + \sum_{j'' \in \hat{\mathcal{N}}_{i'} \setminus \{j\}} H_{t,j''}^{(0)}
\end{align}

This is a system of 2 equations in $\mathbb{R}^d$ with $(|\hat{\mathcal{N}}_i| + |\hat{\mathcal{N}}_{i'}| - 1)$ unknown vectors. For any $|\hat{\mathcal{N}}_i| + |\hat{\mathcal{N}}_{i'}| > 3$, the system is under-determined.

Even in the degenerate case $|\hat{\mathcal{N}}_i| = |\hat{\mathcal{N}}_{i'}| = 1$ (both corrections encode only $j$), the colluders recover $H_{t,j}^{(0)} = X_{t,j}$, the raw feature of node $j$, but learn \textbf{nothing} about $j$'s internal edges $\{(j, u) : u \in \mathcal{N}^{(m')}(j)\}$, since $\delta_i^{(1)}$ contains only 0-hop features which carry no topological information.

For the 2-hop case, even isolating $H_{t,j}^{k,(1)}$ from the summed corrections requires solving an under-determined linear system, and by Sub-claim~2a in Proposition~\ref{prop:2}, a single embedding still does not reveal the internal topology of $G_t^{(m')}$.
\end{proof}

\begin{proposition}[Dimensional Insufficiency]
\label{prop:4}
The information received by client $m$ about client $m'$ is bounded by $O(|B^{(m)}| \cdot d)$, whereas reconstructing $G_t^{(m')}$ requires $\Omega(|V_t^{(m')}|^2)$ bits.
\end{proposition}

\begin{proof}
For each border node $i \in B^{(m)}$, the client receives at most two $d$-dimensional vectors ($\delta_i^{(1)}$ and $\delta_i^{(2)}$). Hence the total information is at most $2|B^{(m)}| \cdot d$ real-valued numbers.

The adjacency matrix $\text{Adj}_t^{(m')} \in \{0,1\}^{|V_t^{(m')}| \times |V_t^{(m')}|}$ has $\binom{|V_t^{(m')}|}{2}$ free entries. From Appendix~\ref{appen:ne_comcost} (Table~\ref{tab:comcost_param}), the cut-edge ratio is extremely low (e.g., 0.0006 for DBLP5), so $|B^{(m)}| \ll |V_t^{(m')}|$. Therefore:
\begin{equation}
    2|B^{(m)}| \cdot d \;\ll\; \binom{|V_t^{(m')}|}{2}
\end{equation}

Even treating each received real number as carrying maximal information, the dimensionality of the received signal is fundamentally insufficient to specify the adjacency structure of $G_t^{(m')}$.
\end{proof}

\section{Reproducibility}
We have open-sourced DG-CoLearn, available at:
\url{https://anonymous.4open.science/r/DG-CoLearn-AB95}

\section{Limitations}
While DG-CoLearn delivers strong empirical results under the client-oblivious trusted-server regime, several limitations are worth noting. First, our privacy guarantees (Appendix H) rely on the server faithfully executing the protocol; extending the framework to handle a malicious server would require integrating complementary techniques such as secure aggregation. Second, we restrict cross-client message passing to 2-hop communication following standard practice [32], which may limit accuracy in graphs where information from deeper neighbourhoods is essential. Third, our scalability evaluation focuses on TGB benchmarks up to 27M edges; extension to billion-edge graphs may require further engineering optimisations such as hierarchical partitioning. Finally, while DG-CoLearn is backbone-agnostic, our experiments with continuous-time encoders are limited to DyGFormer; broader integration with other CTDG methods is left to future work.



\end{document}